\documentclass[Afour,sageh,times]{sagej}

\DeclareMathOperator*{\argmin}{arg\,min}
\DeclareMathOperator*{\argmax}{arg\,max}

\newcommand{\R}{\ensuremath{\mathbb{R}}} 

\newcommand{\angled}[1]{\left\langle#1\right\rangle}
\newcommand{\sqBrackets}[1]{\left[#1\right]}

\newcommand{\task}[0]{\ensuremath{a}}

\newcommand{\allocation}[0]{\ensuremath{\textbf{A}}}  
\newcommand{\allocationSpace}[0]{\ensuremath{\boldsymbol{\mathcal{A}}}}

\newcommand{\taTaskNetwork}[0]{\ensuremath{T}}

\newcommand{\setInitialConfigurations}[0]{\ensuremath{I_c}}

\newcommand{\setInitialTerminalConfigurationsTaskNetwork}[0]{\ensuremath{L_\taTaskNetwork}}

\newcommand{\schedule}[1][none]{\possibleBold{#1}{\ensuremath{\sigma}}}

\newcommand{\scheduleSolution}[0]{\hat{\schedule}}

\newcommand{\worldState}[0]{\ensuremath{W}}

\newcommand{\allocationSolution}[0]{\ensuremath{\hat{\allocation}}}
\newcommand{\mpSolution}[0]{\ensuremath{\hat{\setMotionPlans}}}

\newcommand{\solutionNode}{\hat{N}}

\newcommand{\possibleBold}[2]{%
    \ensuremath{%
        \ifthenelse{\equal{#1}{bold}}%
        {\boldsymbol{#2}}%
        {#2}%
    }%
}

\newcommand{\performanceFunction}[1][none]{\possibleBold{#1}{\xi}}
\newcommand{\performanceFunctionSet}[1][none]{\possibleBold{#1}{\Xi}}
\newcommand{\maxMakespan}[1][none]{\possibleBold{#1}{\ensuremath{C}_{max}}}
\newcommand{\taskNetwork}[1][none]{\possibleBold{#1}{\mathcal{T}}}
\newcommand{\robotTraitMatrix}[1][none]{\possibleBold{#1}{Q}}

\newcommand{\numRobots}[1][none]{\possibleBold{#1}{N}}
\newcommand{\numTasks}[1][none]{\possibleBold{#1}{M}}

\newcommand{\problemDomainStatic}[1][none]{\possibleBold{#1}{\mathcal{D}}}

\newcommand{\solutionStatic}[1][none]{\possibleBold{#1}{S}}

\newcommand{\setMotionPlans}[1][none]{\possibleBold{#1}{\ensuremath{X}}}

\newcommand{\allocationOptimal}[0]{\ensuremath{\allocation^*}}

\newcommand{\algoname}[0]{E-ITAGS}

\newtheorem{theorem}{Theorem}

\usepackage{moreverb,url}
\usepackage{xcolor}
\usepackage{mathtools}
\usepackage{amssymb,amsmath,amsthm}
\usepackage[colorlinks,bookmarksopen,bookmarksnumbered,citecolor=red,urlcolor=red]{hyperref}
\usepackage{accents}
\usepackage{appendix}

\newcommand\BibTeX{{\rmfamily B\kern-.05em \textsc{i\kern-.025em b}\kern-.08em
T\kern-.1667em\lower.7ex\hbox{E}\kern-.125emX}}

\def\volumeyear{2025}

\begin{document}

\runninghead{Learning and Optimizing the Efficacy of MRTA}

\title{Learning and Optimizing the Efficacy of Spatio-Temporal Task Allocation under Temporal and Resource Constraints}

\author{Jiazhen Liu\affilnum{1}, Glen Neville\affilnum{1,2}, Jinwoo Park\affilnum{1}, Sonia Chernova\affilnum{1}, Harish Ravichandar\affilnum{1}}

\affiliation{\affilnum{1}Georgia Institute of Technology, USA; \affilnum{2}Zoox, USA}

\corrauth{Jiazhen Liu, Georgia Institute of Technology,
USA.}
\email{jliu3103@gatech.edu}

\begin{abstract}

Complex multi-robot missions often require heterogeneous teams to jointly optimize task allocation, scheduling, and path planning to improve team performance under strict constraints. 
We formalize these complexities into a new class of problems, dubbed \textit{Spatio-Temporal Efficacy-optimized Allocation for Multi-robot systems (STEAM)}. 
STEAM builds upon trait-based frameworks that model robots using their capabilities (e.g., payload and speed), but goes beyond the typical binary success-failure model by explicitly modeling the efficacy of allocations as \textit{trait-efficacy maps}. These maps encode how the aggregated capabilities assigned to a task determine performance. Further, STEAM accommodates spatio-temporal constraints, including a user-specified \textit{time budget} (i.e., maximum makespan). 
To solve STEAM problems, we contribute a novel algorithm named \textit{Efficacy-optimized Incremental Task Allocation Graph Search (E-ITAGS)} that \textit{simultaneously} optimizes task performance and respects time budgets by \textit{interleaving} task allocation, scheduling, and path planning.
Motivated by the fact that trait-efficacy maps are difficult, if not impossible, to specify, \algoname{} efficiently learns them using a \textit{realizability-aware active learning} module.
Our approach is realizability-aware since it explicitly accounts for the fact that not all combinations of traits are realizable by the robots available during learning. 
Further, we derive experimentally-validated bounds on \algoname{}' suboptimality with respect to efficacy.
Detailed numerical simulations and experiments using an emergency response domain demonstrate that \algoname{} generates allocations of higher efficacy compared to baselines, while respecting resource and spatio-temporal constraints. We also show that our active learning approach is sample efficient and establishes a principled tradeoff between data and computational efficiency. 

\end{abstract}


\maketitle

\section{Introduction}
Heterogeneous multi-robot systems (MRS) bring together complementary capabilities to tackle challenges in domains as diverse as agriculture \citep{Roldan2016}, defense \citep{McCook2007}, assembly \citep{Stroupe2005}, and warehouse automation \citep{Baras2019}. Achieving effective teaming in these complex domains often requires reasoning about task allocation (\textit{who})~\citep{Ravichandar2020}, scheduling (\textit{when})~\citep{Matos2021}, motion planning (\textit{how})~\citep{Baras2019}, and intersections of such spatio-temporal problems ~\citep{Neville2021,Messing2022}. 

In this work, we first formalize a novel class of heterogeneous multi-robot spatio-temporal coordination problems, dubbed \textit{Spatio-Temporal Efficacy-optimized Allocation for Multi-robot systems (STEAM)}. STEAM problems are closely related to ST-MR-TA problems defined by the well-established multi-robot task allocation (MRTA) taxonomies ~\citep{Gerkey2004,Nunes2017} (see Sec.~\ref{subsec:related_work:MRTA} for details).
To effectively tackle STEAM problems, we then develop \textit{Efficacy-optimized Incremental Task Allocation Graph Search (\algoname)}\footnote{Open-sourced at: https://github.com/GT-STAR-Lab/Q-ITAGS},
a heuristic-driven \textit{interleaved} search algorithm  (see Fig.~\ref{fig:block_diagram}), designed for \textit{heterogeneous} teams operating under \textit{time budgets}.  


We study and improve three of crucial aspects of STEAM-like problems: 
i) decomposibility of coalition-level tasks, ii) collective influence of robot capabilities on task performance, and iii) task specification. 
Below, we briefly introduce our approach through these three lenses, contextualized within existing approaches.
Please refer to Sec.~\ref{sec:related_work} for a detailed discussion of related works and the specific subsets of assumptions different approaches make. 

First, several methods assume that one can either decompose multi-robot tasks into individual robot subtasks~\citep{Giordani2013,Krizmancic2020} or \textit{enumerate} all possible coalitions~\citep{gombolay2016,Schillinger2018}. However, it is often difficult, if not impossible, to explicitly specify the role of each robot in tasks that involve complex collaboration, and enumerating every capable coalition might be intractable. In contrast, building upon our prior work in \textit{trait-based} task allocation~\citep{Ravichandar2020,Neville2021,Messing2022,srikanthan2022resource,neville-2023-d-itags}, we model robots in terms of capabilities and task requirements in terms of collective capabilities, leading to a \textit{flexible} and \textit{robot-agnostic} framework. As such, \algoname{} neither requires decomposition of tasks nor enumeration of all possible coalitions.

Second, existing approaches tend to assume a \textit{binary success-or-failure model} of task outcomes~\citep{fu2022robust,Messing2022,srikanthan2022resource}. Such a limited model does not apply to several real-world problems which require maximization of performance metrics, as opposed to satisfaction of arbitrary thresholds (e.g., distributed sensing, supply chain optimization, and disaster response). In contrast, we introduce \textit{trait-efficacy maps}, a novel and expressive \textit{continuous} model that maps collective capabilities (traits) to task performance (efficacy), helping encode and optimize the efficacy of allocations.

Third, existing methods often demand that users \textit{explicitly specify} how robot capabilities relate to successful task completion~\citep{Prorok2017,Ravichandar2020,Mayya2021,neville-2023-d-itags}. However, it is well known that humans, while adept at making complex decisions, often fail to explicitly articulate how they do so; in fact, we tend to make things up to justify our decisions~\citep{rieskamp2003people,rieskamp2006ssl}.  
In contrast, \algoname{} does not demand users to specify how collective capabilities relate to task performance. It employs a data-driven \textit{active learning} module to efficiently learn these influences and encodes them in trait-efficacy maps.


\algoname{} solves STEAM problems by \textit{interleaving} the search for task allocations and schedules. Such an interleaving paradigm has been shown to simultaneously optimize for task allocation, scheduling, and path planning without incurring a prohibitive computational overhead~\citep{Messing2022}. To guide \algoname{}' search, we introduce two new heuristics: i) \textit{Time Budget Overrun (TBO)}, which captures the suboptimality of a given schedule anchored against a user-specified threshold on makespan, and ii) \textit{Normalized Allocation Cost (NAC)}, which provides a normalized estimate of the current allocation's efficacy (as predicted by the trait-efficacy maps). \algoname{} employs a convex combination of these two heuristics and exposes a single hyperparameter that allows users to traverse the inherent trade-off between makespan and task performance.   

When learning trait-efficacy maps in the real-world, seemingly subtle factors can dramatically constrain learning. Consider the fact that the collective capabilities that can be allocated to a given task are directly determined and limited by those of the individual robots available during learning. For instance, with a team of three robots that have payloads of 5, 10, and 15 pounds, it would be impossible to form and evaluate a coalition with an exact collective payload of 8 or 18 pounds. In general, that set of all possible collective capabilities that can be realized by a given team is precisely the set of weighted combinations of the individual capabilities, with binary weights ($0$ or $1$). We refer to this important restriction as \textit{realizability} and explicitly account for it. 
We introduce multiple realizability-aware learning strategies that incorporate realizability to varying degrees of approximation, and show that they establish a principled tradeoff between data and computational efficiency. 

In addition to the above empirical benefits, we contribute theoretical insights into \algoname' operation. We derive a bound on \algoname' sub-optimality in terms of allocation efficacy under mild assumptions. 
Notably, our analysis illuminates an inherent trade-off between allocation efficacy and schedule makespan that hinges on a single hyperparameter. This insight and analysis can offer users intuitive guidance in choosing the hyperparameter.
We also demonstrate via experiments that our bound consistently holds in practice. 


In summary, our core contributions include
\begin{itemize}
\vspace{-0.1cm}
    \item A formalism of the spatio-temporal multi-robot task allocation problems that emphasizes optimizing \textit{allocation efficacy} while respecting a \textit{time budget};
    \item An \textit{interleaved} graph search algorithm guided by two \textit{novel heuristics} to solve such problems;
    \item A \textit{realizability-aware active learning} module that effectively learns unknown trait-efficacy maps, balancing data and computational efficiency;
    \item A \textit{suboptimality bound} for our approach with respect to allocation efficacy. 
\end{itemize}


We evaluated \algoname{} extensively in terms of its ability to i) optimize task allocation efficacy, ii) respect time budgets, and iii) learn unknown mappings between collective capabilities and task performance.
We conducted our evaluations both in numerical simulations and within RoboCup Rescue~\citep{kitano2001robocup,Kitano1999,robocupwebsite}, an autonomous search and rescue simulator. 
Our results conclusively demonstrate that \algoname{} outperforms existing methods in terms of allocation efficacy while generating schedules of a similar makespan. We also provide detailed analyses of \algoname{}' incremental search behaviors to find valid task allocations. 
Further, \algoname' realizability-aware active sampling approach achieves distinctive improvement in computational cost while maintaining data efficiency. 

This manuscript is an extended and revised version of a conference paper by~\cite{q-itags}. Significant additions and improvements include: i) realizability-aware active-learning method, 
which accounts for resource constraints of a given team,
ii) identification and treatment of idiosyncrasies that arise when active learning is utilized for constrained multi-robot coordination,
and 
iii) two sets of additional experiments validating our overall approach in an open-source robotics simulator involving search-and-rescue. For easy reference, we introduce the algorithmic details of i) and ii) in Sec.~\ref{subsec:active_learning_approach}, and the experiment results from iii) in Sec.~\ref{subsec:evaulate_active_learning} and Sec.~\ref{subsec:holistic_eval}. 

In what follows, we review related prior works in Sec.~\ref{sec:related_work}. Sec.~\ref{sec:problemdef}  provides a formal definition of STEAM problems, as well as active learning's role in it.  We introduce \algoname{} in Sec.~\ref{sec:mainapproach}, which includes a derivation of the suboptimality bound in Sec.~\ref{sec:theorems}. 
We introduce our realizability-aware active learning methodology in Sec.~\ref{subsec:active_learning_approach}. Sec.~\ref{sec:evaluations} discusses all of our experiments and results. Finally, we summarize our conclusions and discuss future work in Sec.~\ref{sec:conclusion}. 

\begin{figure*}[t]
\begin{center}
    \includegraphics[width=0.7\textwidth]{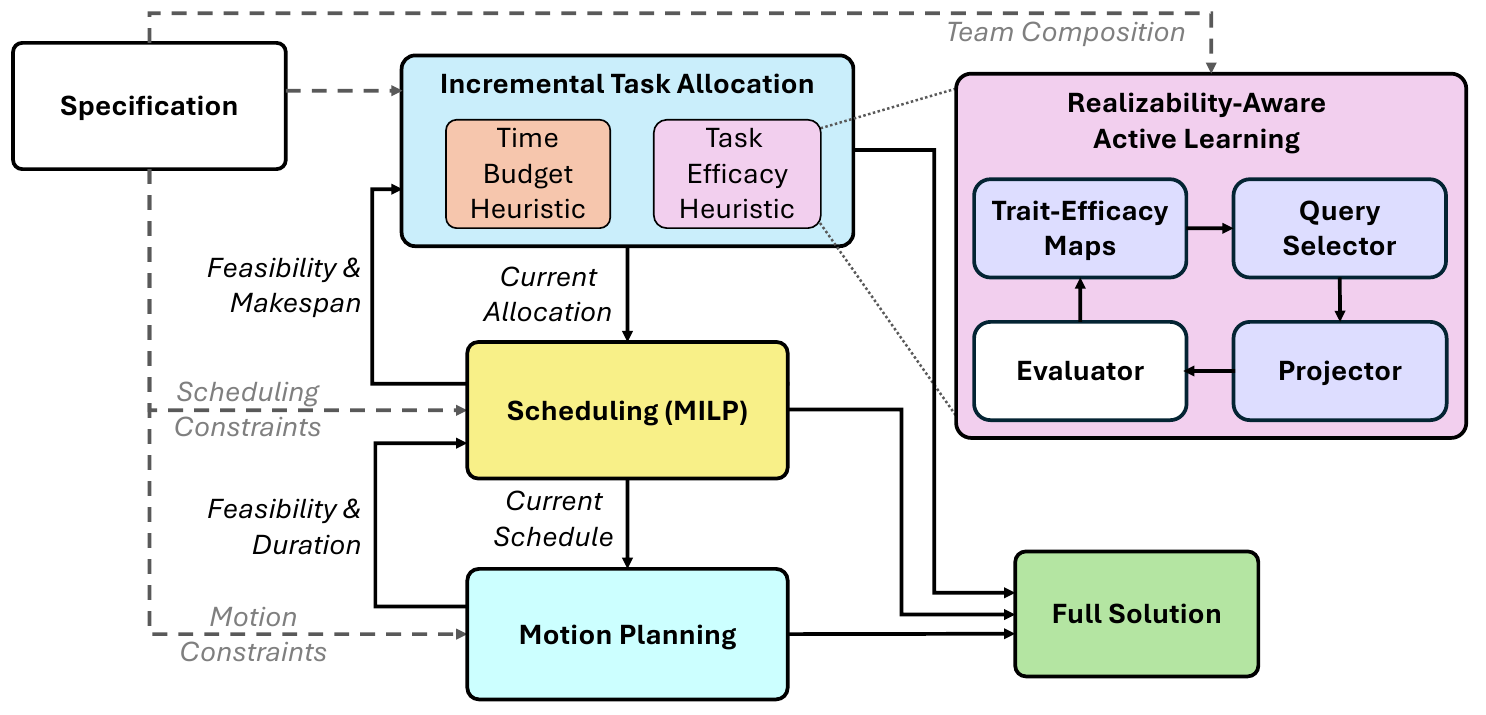}
\end{center}
\caption{
\small{
\algoname{} performs spatio-temporal task allocation for heterogeneous multi-robot teams by optimizing collective performance while respecting spatio-temporal and resource constraints. The realizability-aware active learning module explicitly models, actively learns, and optimizes trait-efficacy maps that approximate the effects of collective capabilities on task performance, with more details in Secs.~\ref{sec:mainapproach} and \ref{subsec:active_learning_approach}, and in Figs.~\ref{fig:sampling_relaxation}. and~\ref{fig:active_learning_pipeline}.
}}
\label{fig:block_diagram} 
\end{figure*}

\section{Related Works}
\label{sec:related_work}
In this section, we survey the related works, which mainly revolve around two pillars: multi-robot task allocation and active learning for multi-robot systems.

\subsection{Multi-Robot Task Allocation}\
\label{subsec:related_work:MRTA}

While the multi-robot task allocation problem has many variants~\citep{Gerkey2004,Nunes2017}, we limit our discussion to one variant closely related to our STEAM problem: single-task (ST) robots, multi-robot (MR) tasks, and time-extended (TA) allocation. ST-MR-TA problems require assigning coalitions to tasks while respecting temporal constraints (e.g., precedence and ordering), spatio-temporal constraints (e.g., travel times), and deadlines. In addition to these issues, our problem formulation also involves spatial constraints (e.g., task locations and obstacles). 

\textit{Auction-based methods} to solve ST-MR-TA problems involve auctioning tasks to robots through a bidding process based on a utility function that combines the robot's (or the coalition's) ability to perform the task with any temporal constraints \citep{Giordani2013,Krizmancic2020,quinton2023market}. Auctions are highly effective, but typically require each robot to know its individual utility towards a task. As a result, these methods require either 
i) multi-robot tasks to be decomposable into sub-tasks, each solvable by a single robot, 
or ii) robots within a coalition to have explicit knowledge of their individual contribution to all collaborative tasks.

\textit{End-to-end optimization-based methods} formulate the ST-MR-TA problem as a monolithic mixed-integer linear program (MILP) that incorporates complex inter-dependency constraints and optimizes the overall makespan or a utility function as the objective~\citep{Korsah2012,choudhury2022dynamic,Guerrero2017,chakraa2023optimization,Schillinger2018}.
However, these methods often require that all tasks to be decomposable into single-agent tasks, such as in~\cite{Schillinger2018}. In contrast to auction-based and optimization-based approaches, our approach does not require task decomposition. 

\textit{Learning-driven methods} for MRTA emerge as a practical alternative paradigm~\citep{10854527,wang2022heterogeneous,10610342,11044426}. Learning-based methods train an allocation policy that outputs assignment decisions and generally achieves higher scalability at inference time.~\cite{10610342} focuses on enabling the robots to voluntarily wait for teammates to become available. The policy is trained using imitation learning on a large-scale dataset with poor sample efficiency.~\cite{10854527,11044426} instead adopts reinforcement learning to circumvent the need for abundant demonstration data. However, both rely on defining the set of minimally necessary resources for tasks, which can be arbitrarily tricky to specify in realistic settings.

Our approach to task allocation is most closely related to \textit{trait-based methods}~\citep{Koes2006,Prorok2017,Ravichandar2020,Neville2020,Neville2021,srikanthan2022resource,fu2022learning,fu2022robust,CMTAB,neville-2023-d-itags,10341837}, which utilizes a flexible modeling framework that encodes task requirements in terms of traits (e.g., Task 1 requires traveling at 10m/s while transporting a 50-lb payload). Each task is not limited to a specific set or number of agents. Instead, the focus is on finding a coalition of agents that \textit{collectively} possess the required capabilities. As such, these methods allow for flexible and automated determination of coalition sizes and compositions.  However, most existing trait-based approaches are limited to single-task robot, multi-robot task, instantaneous allocation (ST-MR-IA) problems that do not require scheduling~\citep{Koes2006,Prorok2017,Ravichandar2020,Neville2020,srikanthan2022resource,CMTAB}, with a few notable exceptions, such as~\cite{Neville2021,neville-2023-d-itags,fu2022robust,fu2022learning}, that can handle ST-MR-TA problems.  


A key limitation of existing trait-based algorithms, including~\cite{Neville2021,neville-2023-d-itags,fu2022robust,10341837}, is that they assume that the user will explicitly quantify the minimum amount of collective traits necessary to perform each task. 
This presents two issues.
First, explicit quantification of minimum requirements can be intractable even for experts when dealing with complex tasks and multi-dimensional capabilities~\citep{rieskamp2003people}. 
Other sub-fields within robotics recognize this concern and have developed methods to prevent and handle reward or utility mis-specification~\citep{Menell2017,Mallozzi2018}.
Second, all methods, including those that learn requirements~\citep{fu2022learning} assume that any and all additional traits allocated beyond the minimum provide no benefit to the team and that not satisfying the specified threshold will lead to complete failure. This effectively ignores the natural variation in task performance as a function of allocated capabilities.
In stark contrast, \algoname{} does not require the user to specify trait requirements and utilizes a more expressive model to capture the relationship between allocated capabilities and task performance.

\subsection{Active Learning for Multi-Robot Systems}
We synthesize an overview of prior works applying active learning in the context of multi-robot or multi-agent systems, categorizing them based on whether the coordination among the robots happens at the motion-level or the task-level. 

Considering coordination at the motion level, active learning is a common paradigm for solving multi-robot active perception and informative path planning~\citep{9632368,kailas1,8260881,tzes2022graph,stadler2023approximating,10107979,9144385}. These problems require sequential decision-making to gather critical information about an unknown environmental process with the robots' onboard sensors. With an increasing amount of information collected, robots can reduce the uncertainty in their knowledge of the environment. The decision on when and where to sample constitutes the ``activeness'' of the learning procedure. 

Applying active learning to multi-robot problems where coordination is reasoned at the task level receives less attention. 
Among the scarce efforts, multi-agent task assignment is formulated as a restless bandit problem in~\cite{le2006multi} with switching costs and rewards.~\cite{CMTAB} examines multi-robot \textit{instantaneous} task allocation with unknown reward functions, which is closely related to our problem setting. It proposes an algorithm that builds upon GP-UCB~\citep{GP-UCB} and utilizes the adaptive discretization mechanism called the zooming algorithm~\citep{slivkins2019introduction}. Though demonstrated to be effective,~\cite{CMTAB} restricts itself to tackling the case where all tasks are assumed to occur concurrently.

\section{Problem Formulation}
\label{sec:problemdef}

We begin with preliminaries from our prior work~\citep{Neville2021,Messing2022,neville-2023-d-itags} for context, and then formalize a new class of spatio-temporal task allocation problems.

Consider a team of \numRobots\ heterogeneous robots, with the $i$th robot's capabilities described by a collection of traits $q^{(i)} = \sqBrackets{q_1^{(i)},\ q_2^{(i)},\ \cdots, q_U^{(i)}}^\intercal \in \mathbb{R}_{\geq 0}^U$, where $q_u^{(i)}$ corresponds to the $u^{th}$ trait for the $i^{th}$ robot. We assign $q_u^{(i)} = 0$ when the $i^{th}$ robot does not possess the $u^{th}$ trait (e.g. firetrucks have a water capacity, but other robots may not).
As such, the capabilities of the team can be defined by a \textbf{team trait matrix} 
$
  \robotTraitMatrix[bold]  = \sqBrackets{q^{(1)},\ \cdots,\ q^{(N)}}^{\intercal} \in \R_{\geq 0}^{N \times U}
$ 
whose $iu$-th element (i.e. $i$th row and $u$th column) denotes the $u^{th}$ trait of the $i^{th}$ robot.

We model the set of $M$ tasks to be completed as a \textbf{Task Network} \taskNetwork[bold]: a directed graph $G=(\mathcal{V}, \mathcal{E})$, with vertices $\mathcal{V}$ representing a set of tasks $\{\task_m\}_{m=1}^{M}$, and edges $\mathcal{E}$ representing constraints between tasks. For a pair of tasks $\task_i$ and $\task_j$ with $a_i \neq a_j$, we consider two kinds of constraints: (i) \textbf{precedence constraint} $\task_i \prec \task_j$ requires that Task $\task_i$ should be completed before the Task $\task_j$ can begin (e.g., a fire must be put out before repairs can begin) \citep{Weld1994}; (ii) \textbf{mutex constraint} ensures that $\task_i$ and $ \task_j$ are not concurrent (e.g. a manipulator cannot pick up two objects simultaneously) \citep{Bhargava2019}. Further, we define \maxMakespan~as the 
\textbf{total time budget} which encodes the longest makespan acceptable to the user.  

Let $\allocation \in \allocationSpace \subseteq \{0,1\}^{M \times N}$ denote the binary \textbf{allocation matrix}, where $\allocation_{mn} = 1$ if and only if the $n$-th robot is allocated to the $m$-th task and $\allocation_{mn} = 0$ otherwise. Robots can complete tasks individually or collaborate as a coalition, and any robot can be allocated to more than one task as long as the tasks are scheduled to be carried out during non-overlapping intervals. We use $\allocation_m$ to denote the $m$-th row of $\allocation$ and it specifies the coalition of robots assigned to the $m$-th task. We further define 
$\mathbf{Y} = \allocation \robotTraitMatrix[bold] \in \R^{M\times U}_{\geq 0}$ to be the aggregated traits for all tasks according to allocation plan $\allocation $, with $y_m$ denoting its $m$-th row containing the collective traits available for the $m$-th task. Note that we use $\mathcal{Y} = \R^{M\times U}_{\geq 0}$ for brevity in upcoming sections. 

\subsection{Efficacy-Optimized Spatio-Temporal Task Allocation}
In this work, we extend the problem formulation from our prior work to account for the efficacy of task allocation. We formulate a new class of problems, \textit{Spatio-Temporal Efficacy-optimized Allocation
for Multi-robot systems (STEAM)}, which involves optimizing the efficacy of robots' assignments to tasks while ensuring that the associated makespan is less than a user-specified threshold. We then extend STEAM to include \textit{active learning} of unknown trait-efficacy maps, which map the collective capabilities of a given coalition to task performance.

Let $\performanceFunction_m: \R^{U}_{\geq 0} \rightarrow [0,1]$ be the normalized \textbf{trait-efficacy map} that computes a non-negative efficacy score (with numbers closer to 1 indicating more effective coalitions) associated with the $m^{th}$ task given the collective traits allocated to it. In essence, $\performanceFunction_m$ quantifies how a given coalition's capabilities translate to performance on the $m^{th}$ task. We define the \textbf{total allocation efficacy} to summarize the performance of all tasks:
\begin{equation}
    \performanceFunctionSet(\allocation) = \sum_{m=1}^{M} \performanceFunction_m(\robotTraitMatrix[bold]^T\allocation_m^T) = \sum_{m=1}^{M}  \performanceFunction_m(y_m) 
    \label{eq:total_allocation_efficacy}
\end{equation}
Note that the maps $\performanceFunction_m, \forall m$ are between \textit{traits} (not robots) and performance, yielding more generalizable relationships. We also discuss the active learning of trait-efficacy maps in Sec.~\ref{subsec:active_learning_staq} when they are unknown. 

We define the \textbf{problem domain} using the tuple $\problemDomainStatic[bold] = \big<\taskNetwork[bold],\ \robotTraitMatrix[bold],\ \performanceFunctionSet,\ \setInitialConfigurations,\ \setInitialTerminalConfigurationsTaskNetwork,\ \worldState{},\ \maxMakespan \big>$,
where
\taskNetwork[bold]\ is the task network, \robotTraitMatrix[bold]\ is the team trait matrix, \performanceFunctionSet\ is the summarized performance function, $\setInitialConfigurations$ and $\setInitialTerminalConfigurationsTaskNetwork$ are respectively the sets of all initial and terminal configurations associated with tasks, \worldState{} is a description of the world (e.g., a map), and \maxMakespan~is the total time budget.

Finally, we define a \textbf{solution} to the problem defined by $\problemDomainStatic[bold]$ using the tuple $\boldsymbol{\solutionStatic} = \angled{\allocationSolution,\ \scheduleSolution,\ \mpSolution}$, where $\allocationSolution$\ is an allocation, $\scheduleSolution$\ is a schedule that respects all temporal constraints, and $\mpSolution$\ is a finite set of collision-free motion plans. The goal of STEAM is to find a solution that maximizes the total allocation efficacy $\performanceFunctionSet(\allocation)$ while respecting the total time budget $\maxMakespan$.

\subsection{Realizability-Aware Active Learning}
\label{subsec:active_learning_staq}
As above, $\performanceFunction_m(\cdot)$ is a function that maps the collective traits assigned to the $m$th task to a measure of task performance. While accurate knowledge of these maps is needed to make effective trade-offs when allocating limited resources (i.e., robots) to competing objectives (i.e., tasks), they are often difficult to explicitly specify. As such, we do not assume knowledge of $\performanceFunction_m(\cdot),\ \forall m$. Instead, we assume access to a simulator from which to learn. We are specifically interested in limiting the number of queries made to facilitate learning since such queries can be expensive or time-consuming. As such, we formulate the problem of learning $\performanceFunction_m(\cdot),\ \forall m$ as an active learning problem in which one must effectively sample promising coalitions to efficiently learn the trait-efficacy maps. 

To learn $\{\performanceFunction_m(\cdot) \}_{m=1}^{M}$ actively, we utilize a query selector to sample the most promising data point, as dictated by the acquisition function $V: \mathcal{Y} \rightarrow \mathbb{R}$. The acquisition function quantifies the utility of samples, and common definitions include maximizing the upper confidence bound~\citep{GP-UCB} or maximizing the variance.  
A noticeable roadblock for effective sampling is that a point $\mathbf{Y} \in \mathcal{Y}$, believed to have a high utility by $V$, is not necessarily \textit{realizable} for a given set of robots with team trait matrix $\robotTraitMatrix[bold]$. This stems from the fact that each robot is a non-divisible unit and can either be assigned or not assigned to a task. In fact, the truly directly realizable samples are the ones satisfying $\mathbf{Y}=\mathbf{\allocation \robotTraitMatrix[bold]}$, for some $\allocation$ in the allocation space $\allocationSpace \subseteq \{0,1\}^{M \times N}$. Formally, we define $\mathcal{Y}_{\robotTraitMatrix[bold]} = \{ \mathbf{Y} \in \mathcal{Y} | \mathbf{Y} = \allocation \robotTraitMatrix[bold], \forall \allocation \in \allocationSpace\}$ as the set of all directly realizable samples. We take into account such disconnect between the common design practice of acquisition functions and the samples' realizability. To bridge the disjointedness, we aim to use samples' realizability to inform the acquisition function. 

\section{Efficacy Optimization under Time Budgets}
\label{sec:mainapproach}

To solve the STEAM problems as defined in Sec. \ref{sec:problemdef}, 
we introduce \textbf{Efficacy-optimized Iterative Task Allocation Graph Search (\algoname)} algorithm. In this section, we focus on solving STEAM problems online and assume that the trait-efficacy maps $\{\performanceFunction_m(\cdot) \}_{m=1}^{M}$ have already been learned. We discuss the learning of trait-efficacy maps in Sec.~\ref{subsec:active_learning_approach}.

We begin with a brief overview of \algoname{} before supplying details. \algoname{} utilizes an interleaved architecture to simultaneously solve task allocation, scheduling, and motion planning (see Fig.~\ref{fig:block_diagram}). This interleaved approach is inherited from our prior works~\cite{Neville2021,Messing2022,neville-2023-d-itags}, in which we provide extensive evidence for its benefits over sequential counterparts that are commonly found in the literature. For instance, interleaved architectures are considerably more computationally efficient since they avoid backtracking~\citep{Messing2022}.
However, \algoname{} makes crucial improvements over our prior work: i) optimizing for allocation efficacy instead of makespan, ii) utilizing a more expressive model of task performance, and iii) respecting a time budget (see Sec.~\ref{sec:related_work}).

\algoname{} performs an incremental graph-based search that is guided by a novel heuristic. Our heuristic balances the search between optimizing task allocation efficacy and meeting the time budget requirement. We formulate and solve a mixed-integer linear program (MILP) to address the scheduling and motion planning problems as part of the incremental search. To get around the challenge of explicitly specifying task requirements, \algoname{} employs an \textit{active learning} module to learn the relationship between the collective traits of the coalition and task performance. 
When a very tight time budget $C_{max}$ is imposed, it might be possible that no feasible solution exists no matter how the robots are allocated. For such cases, \algoname{} can signal infeasibility after it exhausts all combinations through the search without finding a solution. 

\subsection{Task Allocation}\label{subsec:allocation} 
The task allocation layer performs a greedy best-first search through an incremental task allocation graph $\mathcal{N}$. In this graph, each node represents an allocation of robots to tasks. Nodes are connected to other nodes that differ only by the assignment of a single robot. 
Note that the root node represents all robots allocated to each task.  Indeed, such an allocation most often will result in the best performance (since all available robots contribute to each task), but will result in the longest possible schedule (sequential execution of tasks). \algoname' search starts from this initial node, and incrementally removes assignments to find an allocation that can meet the time budget without significantly sacrificing task performance.

To guide the search, we developed two heuristics: \textit{Normalized Allocation Cost (NAC)}, which guides the search based on the efficacy of the allocation, and \textit{Time Budget Overrun (TBO)}, which guides the search based on the makespan of the schedule associated with the allocation. 

\noindent \textbf{\textit{Normalized Allocation Cost (NAC)}} is a normalized measure of how a given allocation improves task performance:
\begin{equation}
    \label{equ:NAC}
    f_{NAC}(\overline{\allocation})  = \frac{\performanceFunctionSet(\allocation_\text{root}) - \performanceFunctionSet(\overline{\allocation})}{\performanceFunctionSet(\allocation_\text{root}) - \performanceFunctionSet(\allocation_\text{null})}    
\end{equation}
where $\overline{\allocation}$ is the allocation being evaluated, $\performanceFunctionSet(\overline{\allocation})$ is the total allocation efficacy as defined in Eq.~(\ref{eq:total_allocation_efficacy}), $\allocation_\text{root}$ is the allocation at the root node, and $\allocation_\text{null}$ represents no robots being assigned to any task. Since $\performanceFunctionSet(\allocation_\text{root})$ and $\performanceFunctionSet(\allocation_\text{null})$ respectively define the upper and lower bounds of allocation efficacy, NAC is bounded within $[0,1]$. Note that the efficacy functions $\{\performanceFunction_m(\cdot)\}_{m=1}^{M}$, needed to compute $\performanceFunctionSet(\cdot)$, would be learned as described in Sec.~\ref{subsec:active_learning_approach}. We term this heuristic the ``cost'' since Eq.~\ref{equ:NAC} essentially calculates a normalized decrease in task performance, i.e., a cost to pay, when we subtract assignments from the root node.  

NAC does not consider the scheduling layer and promotes a broader search of the allocation graph. This is due to the fact that shallower nodes have more robots allocated and as a result tend to result in higher allocation efficacy. As such, NAC favors allocations that maximize allocation efficacy and task performance at the expense of a potentially longer makespan.

\vspace{3pt}
\noindent \textbf{\textit{Time Budget Overrun (TBO)}} measures how much the schedule associated with a given allocation violates the time budget, and is calculated as 
\begin{equation}
    \label{equ:pos}
    f_{TBO}(C_{\overline{\schedule[bold]}}) = \max \left( \frac{C_{\overline{\schedule[bold]}} -  C_{max}}{\vert C_{\schedule[bold]_{worst}} -  C_{max}\vert},\ 0 \right)    
\end{equation}
where $C_{\schedule[bold]}$ denotes the makespan of the schedule $\schedule[bold]$, $\overline{\schedule[bold]}$ is the schedule associated with the $\overline{\allocation}$ being evaluated, $C_{max}$ is the user-specified time budget for makespan, and $\schedule[bold]_{worst}$ is the longest schedule which allocates all robots to each task.

Since TBO only considers the schedule and not the allocation, it tends to favor nodes deeper in the graph that have fewer robots and constraints and thus a lower makespan. 
As such, TBO favors allocations that satisfy the time budget at the expense of a deeper search and more node expansions.

\vspace{3pt}
\noindent \textbf{\textit{Time-Extended Task Allocation Metric (TETAM)}} is a convex combination of NAC and TBO, balancing allocation efficacy and time budget:
\begin{equation}
    \label{equ:TETAM}
    f_{TETAM}(\overline{\allocation}, C_{\overline{\schedule[bold]}})= (1-\alpha)f_{NAC}(\overline{\allocation})  + \alpha f_{TBO}(C_{\overline{\schedule[bold]}})
\end{equation}
where $\alpha \in [0,1]$ is a user-specified parameter that controls each heuristic's relative influence. 
Thus, TETAM considers both allocation efficacy and the schedule simultaneously when searching for a solution. See Sec.~\ref{sec:theorems} for a theoretical analysis of the trade-offs between allocation efficacy and makespan.

\subsection{Scheduling and Motion Planning} \label{subsec:scheduling_and_MP} 
\algoname{}' scheduling layer checks the feasibility of scheduling a given allocation and helps compute its TBO. To this end, we formulate a mixed-integer linear program (MILP) that considers three types of temporal constraints: precedence, mutex, and travel time. Precedence constraints $\mathcal{P}$ ensure that one task happens before another (e.g., the fire must be doused before repairs begin). Mutex constraints $\mathcal{M}$ ensure that two tasks do not happen simultaneously (e.g., a robot cannot pick up two objects simultaneously). Travel time constraints ensure that robots have sufficient time to travel between task sites (e.g., traveling to the location of fire before dousing). The MILP takes the following form:

\begin{equation*}
\begin{aligned}
    \min_{\{s_i\}, \{ p_{ij}\}}  &\ C \\
    \mathrm{s.t.}    &\ C \geq s_i + d_i,\ \forall i=1,..,\numTasks \\
                     &\ s_j \geq s_i + d_i + x_{ij},\ \forall (i,j) \in \mathcal{P} \\
                     &\ s_i \geq x_i,\ \forall i=1,..,\numTasks \\
                     &\ s_i \geq s_j + d_j + x_{ji} - \beta p_{ij},\ \forall (i,j) \in \mathcal{M}^R \\
                     &\ s_j \geq s_i + d_i + x_{ij} - \beta(1-p_{ij}),\ \forall (i,j) \in \mathcal{M}^R
\end{aligned}    
\end{equation*}
where $C$ is the makespan, $s_i$ and $d_i$ are the (relative) starting time and duration of Task $\task_i$, $x_{ij}$ is the time required to translate from the site of $\task_i$ to the site of $\task_j$, $x_i$ is the initial time required for the allocated coalition to travel to the site of Task $a_i$, $p_{ij} \in \{ 0, 1\}$, and $p_{ij} = 1$ if and only if $\task_i$ precedes $\task_j$, $\beta \in \mathbb{R}_+$ is a large scalar, and $\mathcal{P}$ and $\mathcal{M}$ are sets of integer pairs containing the lists of precedence and mutex constraints, with $\mathcal{M}^R = \mathcal{M} - \mathcal{P} \cap \mathcal{M}$ denoting mutex constraints with precedence constraints removed.

\algoname{} initializes the schedule by estimating travel times using Euclidean distances. During the search, the scheduling layer iteratively queries the motion planner to better estimate travel times. \algoname{} iterates until all motion plans required by the schedule are instantiated and memoized.

\subsection{Suboptimality Bounds of Allocation}
\label{sec:theorems}

To better understand \algoname' performance, we analyze the effect of $\alpha$ -- the user-specified parameter that determines the relative importance of our two heuristics -- on the optimality of the obtained solution defined with respect to the efficacy of task allocation. We demonstrate that 
the choice of $\alpha$ determines a suboptimality bound, where $\alpha$ values between $\alpha=0$ and $\alpha=0.5$ promise increasingly tighter bounds on suboptimality under mild assumptions. This analysis and the fact that $\alpha$ values closer to 1 provide solutions that prioritize makespan can offer an intuitive understanding of the relationship between allocation efficacy and makespan.

Below, we derive strict suboptimality bounds for solutions generated by \algoname{}, in terms of total allocation efficacy. 


\begin{theorem}\label{thm:makespan}
    For any given problem domain \problemDomainStatic[bold], let $\allocationOptimal$ be the optimal allocation w.r.t. total allocation efficacy, and $\allocationSolution$ be the allocation of the solution generated by \algoname{} for the same problem. Assuming that adding a robot would never reduce a coalition's performance, the suboptimality of \algoname{}' solution with respect to allocation efficacy is given as follows when $\alpha < 0.5$ in Eq. (\ref{equ:TETAM}):
    \begin{equation}
    \label{equ:boundsPre}
         \performanceFunctionSet(\allocationOptimal)- \performanceFunctionSet(\allocationSolution) \leq \frac{\alpha}{1-\alpha} (\performanceFunctionSet(\allocation_\text{root})- \performanceFunctionSet(\allocation_\text{null})),
    \end{equation}
    where $\performanceFunctionSet(\allocation_\text{root})$ and $\performanceFunctionSet(\allocation_\text{null})$ denote the total allocation efficacy respectively when all robots are assigned to all tasks, and when no robot is assigned to any task. Additionally, when $\alpha = 0$, \algoname{} gives an optimal solution and the suboptimality gap vanishes, i.e., \performanceFunctionSet(\allocationOptimal) - \performanceFunctionSet(\allocationSolution) $= 0$. 
    \label{thm1}
\end{theorem}

\begin{proof}

Since any expansion of a parent node represents the subtraction of an assignment, any given node $N$ is guaranteed to have fewer robots assigned than its parent $N_p$. This observation, 
when combined with our assumption that coalition performance can never improve when any one of the assigned robots is taken away (i.e., no robot contributes negatively),
we can see that
 \begin{equation}
 \label{equ:monoNSQ}
    f_{NAC}(N) \geq f_{NAC}(N_p),
 \end{equation}
since a higher allocation efficacy results in a lower NAC heuristic value and vice versa. Consequently, we can infer that the NAC value of all nodes in the unopened set $\mathcal{U} \subseteq \mathcal{N}$ of the \algoname~graph is greater than or equal to that of their respective predecessors in the opened set $\mathcal{O} \subseteq \mathcal{N}$.
As such, the smallest NAC value in the unopened set must be greater than that in the opened set:

 \begin{equation}
 \label{equ:schedIncrease}
     \min_{N \in \mathcal{U}}f_{NAC}(N) \geq \min_{N \in \mathcal{O}}f_{NAC}(N)
 \end{equation}

Irrespective of whether the optimal allocation $\allocationOptimal$ corresponds to a node in the open or unopened set, 
the inequality in Eq.~(\ref{equ:schedIncrease}) implies that
 \begin{equation}
 \label{equ:boundOpt}
\performanceFunctionSet(\allocationOptimal) \leq \max_{N \in \mathcal{O}} \performanceFunctionSet(\allocation_N)
 \end{equation}
where $\allocation_N$ denotes the allocation of a given node $N$. We use $N'$ to denote the node from the open set with the maximum total allocation efficacy. Namely, $N' = \argmax_{N \in \mathcal{O}}\performanceFunctionSet(\allocation_N)$.

\begin{figure*}[t]
    \centering
    \includegraphics[width=0.95\textwidth]{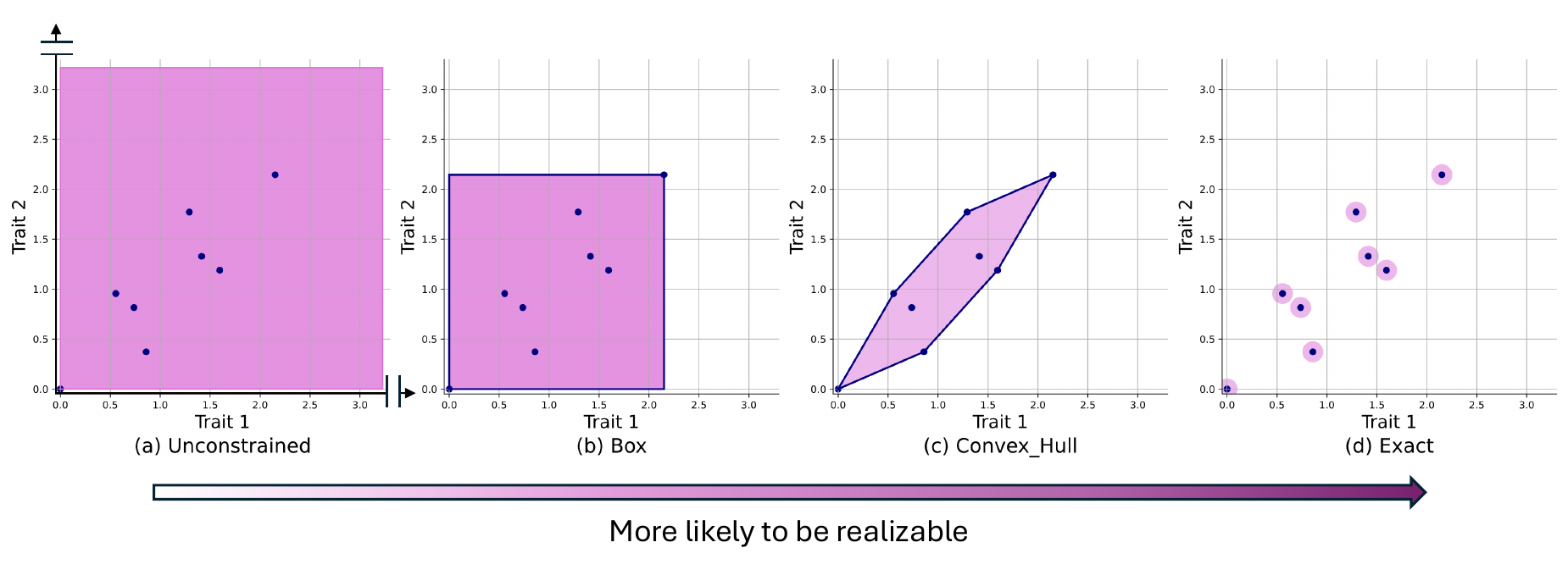}
    \caption{We visualize the subspaces to sample from in all four strategies in an increasing order of samples' realizability, in a scenario with 3 robots, each characterized by 2 traits. For each strategy, the pink region denotes the subspace allowable for sampling when the acquisition function selects additional trait combinations for labeling. Note that in (a), the unconstrained subspace is visualized as a large pink rectangle with axes extending to infinity. To show the exact strategy, we draw a pink circle area around each trait combination only for clarity of demonstration. In fact, only discrete dots in dark blue are considered valid candidates for selection. }
    \label{fig:sampling_relaxation}
\end{figure*}
 
Since we require any valid solution of \algoname{} to respect the time budget, the solution node ($\solutionNode$) will have a zero TBO value. 
As such, the total TETAM heuristic defined in Eq.~(\ref{equ:TETAM}) for the solution node is given by
\begin{equation}
 \label{equ:tetamsol}
      f_{TETAM}(\solutionNode) = (1- \alpha) f_{NAC}(\solutionNode) 
 \end{equation}
Since \algoname{} performs greedy best-first search, the TETAM value of the solution node is guaranteed to be less than or equal to all nodes in the open set at the time we find the solution $\solutionNode$:
 \begin{equation}
 \label{equ:soldef_queue}
    f_{TETAM}(\solutionNode) \leq f_{TETAM}(N), \\ \forall N \in \mathcal{O} 
 \end{equation}
Expanding the definition of TETAM and using Eq.~(\ref{equ:tetamsol}) yields
 \begin{equation}
 \label{equ:simplify}
    \begin{aligned}
    & (1-\alpha)\frac{ \performanceFunctionSet(\allocation_\text{root}) - \performanceFunctionSet(\allocationSolution)}
    {\performanceFunctionSet(\allocation_\text{root})- \performanceFunctionSet(\allocation_\text{null})}
    \leq \\
    & \alpha  f_{TBO}(N) + (1-\alpha)\frac{\performanceFunctionSet(\allocation_\text{root}) - \performanceFunctionSet(\allocation_{N})}
    {\performanceFunctionSet(\allocation_\text{root}) - \performanceFunctionSet(\allocation_\text{null})}, \forall N \in \mathcal{O}  
    \end{aligned}
 \end{equation}
Using the bound in Eq.~(\ref{equ:boundOpt}), and the fact that $ f_{TBO}(\cdot) \leq 1$, we rewrite the RHS of above equation as
 \begin{equation}
    \label{equ:simplify_2}
    \begin{aligned}
            & (1-\alpha)\frac{\performanceFunctionSet(\allocation_\text{root}) -  \performanceFunctionSet(\allocationSolution)}
            {\performanceFunctionSet(\allocation_\text{root}) - \performanceFunctionSet(\allocation_\text{null})}
            \leq  \\
            & \alpha + (1-\alpha)\frac{\performanceFunctionSet(\allocation_\text{root}) - \performanceFunctionSet(\allocationOptimal)}
            {\performanceFunctionSet(\allocation_\text{root}) - \performanceFunctionSet(\allocation_\text{null})
            },  
    \end{aligned}
 \end{equation}
After rearranging and canceling equivalent terms on both sides, we get
 \begin{equation}
 \label{equ:simplify_3}
 \begin{aligned}
    & (1-\alpha)(\performanceFunctionSet(\allocation_\text{root}) - \performanceFunctionSet(\allocationSolution))
    \leq \\
    & \alpha (\performanceFunctionSet(\allocation_\text{root}) - \performanceFunctionSet(\allocation_\text{null})) + (1-\alpha) (\performanceFunctionSet(\allocation_\text{root})) - \performanceFunctionSet(\allocationOptimal))
\end{aligned}
 \end{equation} 
Rearranging the terms yields the bound in Eq.~(\ref{equ:boundsPre}). 

Further, when $\alpha = 0$, Eq.~(\ref{equ:boundsPre}) simplifies to $\performanceFunctionSet(\allocationOptimal) -\performanceFunctionSet(\allocationSolution) \leq 0$. However, since the optimal solution must achieve the highest total allocation efficacy,  $\performanceFunctionSet(\allocationOptimal) - \performanceFunctionSet(\allocationSolution) = 0 $. As such, \algoname{} returns the solution with optimal allocation efficacy when $\alpha = 0$. \qed
\end{proof}

Note that the above theorem provides sensible bounds on the difference in allocation efficacy between the optimal solution and \algoname' solution only when $0 \leq \alpha < 0.5$. When $\alpha \geq 0.5$, the bound loses significance as it grows beyond the maximum difference in efficacy between $\performanceFunctionSet(\allocation_\text{root})$ and $\performanceFunctionSet(\allocation_\text{null})$. In other words, Eq.~\eqref{equ:boundsPre} holds trivially when $\alpha \geq 0.5$, since $\performanceFunctionSet(\allocation_\text{null})$ has the worst efficacy metric of all allocations and $\performanceFunctionSet(\allocation_\text{root})$ has the best one, meaning a bound greater than $\performanceFunctionSet(\allocation_\text{root}) - \performanceFunctionSet(\allocation_\text{null})$ would not be significant.

\textbf{Post-hoc bound}: The bound in Eq.~(\ref{equ:boundsPre}) can be tightened after the execution to facilitate post-hoc analyses. Specifically, instead of bounding $f_{TBO}(N), \forall N \in \mathcal{O}$ by 1, 
we can exactly compute $f_{TBO}(N')$ for $N'$, which is the node yielding the best allocation efficacy within $\mathcal{O}$.
Then, following similar algebraic manipulations as above yields a tighter bound on the suboptimality gap:
\begin{equation*}
    \label{equ:boundsPost}
     \performanceFunctionSet(\allocationOptimal)- \performanceFunctionSet(\allocationSolution) \leq \frac{\alpha}{1-\alpha} (\performanceFunctionSet(\allocation_\text{root}) - \performanceFunctionSet(\allocation_\text{null})) f_{TBO}(N’)
\end{equation*}

\section{Realizability-Aware Active Learning}
\label{subsec:active_learning_approach}

\begin{figure*}
    \centering
    \includegraphics[width=0.9\linewidth]{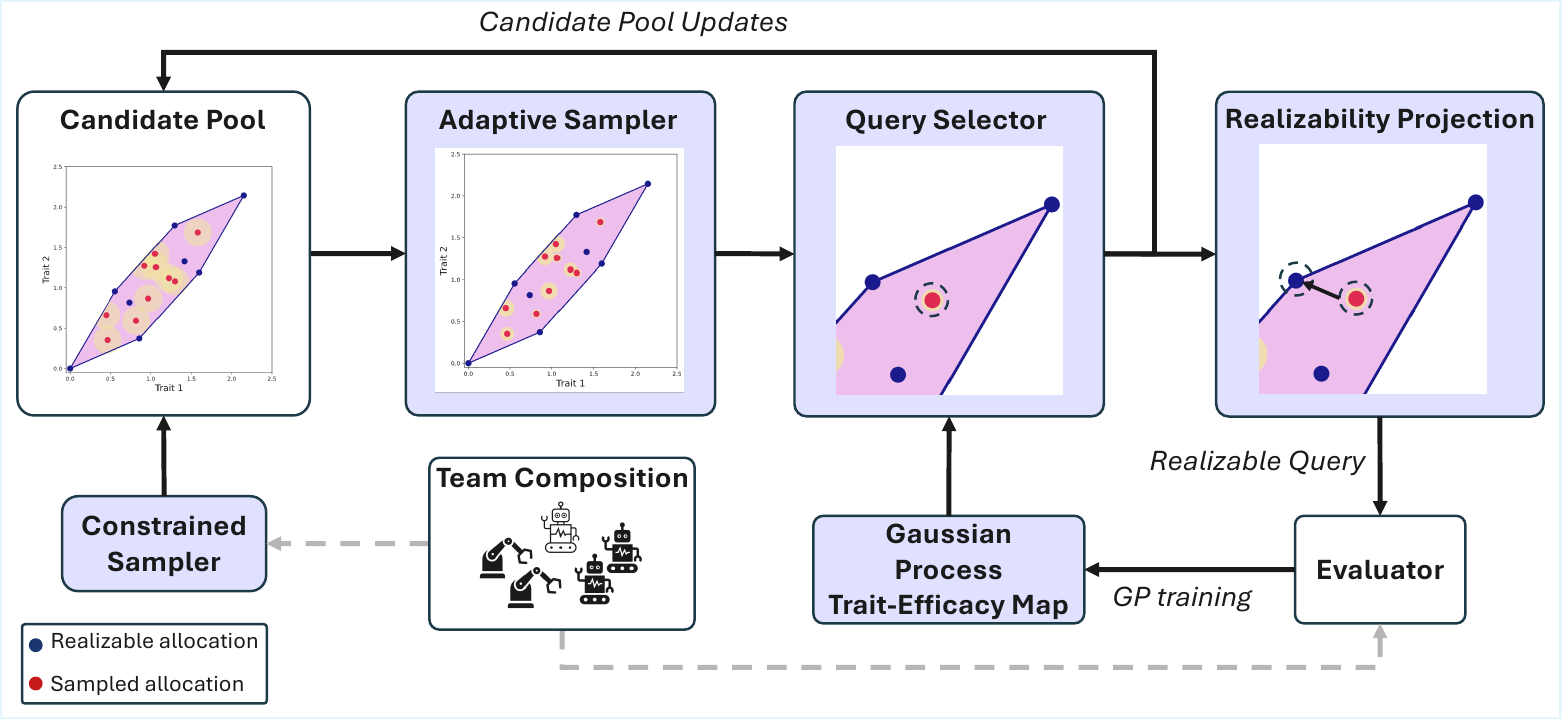}
    \caption{The proposed realizability-aware active learning pipeline. We illustrate our idea using the convex hull-constrained strategy in a simple scenario with 3 robots, each characterized by 2 traits. The box-constrained strategy shares the same pipeline except for its allowable region to sample the candidates. On a 2D plane denoting the trait space, the pink region is encircled by the convex hull of all directly realizable trait combinations in $\mathcal{Y}_{\robotTraitMatrix[bold]}$ (the dark blue dots). Red points within the convex hull are randomly sampled candidates, with each candidate's neighborhood shown as a yellow circle. Note that all neighborhoods are initialized to be the same size, while adaptively shrunk by the zooming mechanism as reflected by smaller yellow circles in the ``Adaptive Sampler''. The UCB-based selector picks the candidate, shown as the dashed circle, followed by the projector projecting it to the nearest neighbor in $\mathcal{Y}_{\robotTraitMatrix[bold]}$. The realizable sample is fed into the simulated evaluator for label, which is leveraged to train the trait-efficacy map. }
    \label{fig:active_learning_pipeline}
\end{figure*}

We now introduce our approach to efficiently learning trait-efficacy maps $\{\performanceFunction_m(\cdot) \}_{m=1}^{M}$. As mentioned in Sec.~\ref{subsec:active_learning_staq}, for each task $m$, its trait-efficacy map $\performanceFunction_m(y_m)$ specifies the relationship between the collective traits $y_m$ allocated to it and the assigned coalition's performance. With the idea of active learning, our goal at each iteration is to sample $y_m,\forall m$. We feed the corresponding allocation matrix $\allocation$ into a simulator for evaluation and labeling. The maps are updated on pairs of collective traits and labels iteratively. 

We model the joint distribution of collective traits assigned to a task and the resulting performance as a Gaussian process (GP), denoted as $P\left(y_m,\performanceFunction_m(y_m)\right) \sim \mathcal{GP}(\mu_m,\Sigma_m), \forall m$, since GP has been shown to effectively model a wide range of functions~\citep{CMTAB,GP-UCB}. This choice allows us to learn the trait-efficacy maps using Gaussian process regression. Correspondingly, we use upper confidence bound (UCB)~\citep{GP-UCB} for evaluating the utility of points in the acquisition function. 

As introduced in Sec.~\ref{subsec:active_learning_staq}, only samples that satisfy $\mathbf{Y} \in \mathcal{Y}_{\robotTraitMatrix[bold]}$ are directly realizable. These samples occupy only a small subspace of the continuous trait space $\mathcal{Y}$ as they are sparsely distributed. 
As such, the intuitive strategy of employing rejection sampling is unreliable, if not infeasible, since it requires checking a prohibitively large number of points in order to acquire sufficient number of realizable samples. 

\textbf{Realizability projection}: One way to overcome the inefficiency of rejection sampling and bridge the gap between the acquisition function and samples' realizability is to use projections. To this end, we first select $\mathbf{Y} \in \mathcal{Y}$ with the highest utility $V(\mathbf{Y})$, and project it onto the realizable set $\mathcal{Y}_{\robotTraitMatrix[bold]}$ by first solving the following optimization,
\begin{equation}
\begin{aligned}
    \allocation' = \argmin_{\allocation \in \allocationSpace}  &\ \lVert \mathbf{Y}-\allocation \robotTraitMatrix[bold]\rVert^2_2, \\
\end{aligned}  
\label{eq:projection_opt}
\end{equation}
which is an integer quadratic program with a quadratic objective. Through solving Eq.~\eqref{eq:projection_opt}, we obtain the allocation matrix that most closely matches the desired aggregated traits $\mathbf{Y}$. Then, the realizable task trait matrix is computed as $\mathbf{Y}'=\allocation'\robotTraitMatrix[bold]$. We call solving Eq.~\eqref{eq:projection_opt} together with computing $\mathbf{Y}'$ the \textit{realizability projection} for the rest of the paper. $\mathbf{A}'$ and $\robotTraitMatrix[bold]$ are fed into the simulator to acquire labels representing the coalitions' performance. With a little abuse of notation, we denote the labels as $\{ \performanceFunction_m \}^{M}_{m=1}$. We train the Gaussian process corresponding to task $m$ with the pair $(\mathbf{y}'_m,{}\performanceFunction_m)$, where $\mathbf{y}'_m$ is the row corresponding to task $m$ in $\mathbf{Y}'_m$. Since we do not impose any constraint when we sample $\mathbf{Y}$ while realizability is purely endowed through projection, we term this strategy as the \textit{unconstrained strategy}.


A prominent failure mode of the unconstrained strategy is that the Gaussian process model may repeatedly rely on data points $\mathbf{Y}'$ that are not necessarily promising (as indicated by low utility values $V(\mathbf{Y}')$) since $\mathbf{Y}'$ can be arbitrarily far away from the original $\mathbf{Y}$. 
This slows down the model from exploiting its knowledge of which parts of the aggregated trait space lead to higher rewards.  

\textbf{Constrained sampling}: In contrast to the unconstrained strategy, one can restrict the pool of candidate aggregated traits that we sample from to only realizable ones in $\mathcal{Y}_{\robotTraitMatrix[bold]}$. Given the team composition matrix \robotTraitMatrix[bold], in each learning iteration, we aim to select $\mathbf{Y}$ via solving
\begin{equation}
\begin{aligned}
    \max_{\mathbf{Y} \in \mathcal{Y}_{\robotTraitMatrix[bold]}} \ V(\mathbf{Y}), 
\end{aligned}
\label{eq:exact_opt}
\end{equation}
which naturally guarantees that $\mathbf{Y}$ is realizable by the robot team and thus does not require the projection step. Since this strategy considers only the set of exactly realizable values for $\mathbf{Y}$, we call it the \textit{exact strategy}. We note that Eq.~\eqref{eq:exact_opt} can be immediately recast as,
\begin{equation}
\begin{aligned}
    \max_{\mathbf{Y},~\allocation\in\allocationSpace} \quad & V(\mathbf{Y}) \\
    s.t. \quad & \mathbf{Y} = \allocation \robotTraitMatrix[bold]
\end{aligned}
\label{eq:exact_opt_recast}
\end{equation}

Eq.~\eqref{eq:exact_opt_recast} is a mixed-integer optimization program (MIP) whose objective function $V(\cdot)$ provides estimates of the utility of variable $\mathbf{Y}$. Common choices of $V(\cdot)$ include maximum entropy and UCB, which rely on the first and second moments of the distribution of trait-efficacy maps. Since we adopt Gaussian processes to model the trait-efficacy maps, evaluating $V(\cdot)$ requires extracting the means and variances from the GP at every iteration \textit{within the optimizer}, which is not readily supported by most off-the-shelf mixed-integer nonlinear program solvers (MINLP).
Therefore, solving Eq.~\eqref{eq:exact_opt_recast} necessitates either a substantial amount of effort to embed the Gaussian process into the MINLP software or designing certain surrogates to approximate $V(\cdot)$ itself. 

\textbf{Balancing realizability and computational efficiency}: To enable realizability-aware sampling without incurring a large computational or logistical burden, we introduce a spectrum of approaches that approximate the \textit{sampling region} to varying degrees. 
Taking inspiration from~\cite{baptista2018bayesian}, we approximate the realizable set $\mathcal{Y}_{\robotTraitMatrix[bold]}$ by sampling from relaxed, continuous subspaces within $\mathcal{Y}$. 

\textit{Box}: We can compute the total amount of traits provided by the whole team given the team trait matrix $\robotTraitMatrix[bold]$ by summing the rows of it: $\bar{\mathbf{q}} = \mathbf{1}_N^\top\robotTraitMatrix[bold]  \in \mathbb{R}^{U}_{\geq 0}$, where $\mathbf{1}_N$ is an all-one vector of dimension $N$ here. We denote the summation result by a vector $\bar{\mathbf{q}} \in \mathbb{R}^{U}_{\geq 0}$. As such, $\bar{\mathbf{q}}$ serves as a perfect \textit{upperbound} when we sample $\mathbf{Y}$. We also define a zero vector $\mathbf{q}_0=\mathbf{0}\in \mathbb{R}_{\geq 0}^{U}$ as the trivial lower bound. We constrain the region of sampling to the box-shaped subspace bounded between $\bar{\mathbf{q}}$ and $\mathbf{q}_0$, denoted as $\mathcal{B}$. We call this strategy, which only samples from $\mathcal{B}$, as \textit{box-constrained strategy}. 
Selecting $\mathbf{Y}$ in this case is framed as solving 
\begin{equation*}
\begin{aligned}
    \max_{\mathbf{Y} \in \mathcal{B}}  &\ V(\mathbf{Y}) \\
\end{aligned}
\label{eq:box_opt}
\end{equation*}

\textit{Convex hull}: We can further obtain a tighter (but still continuous) approximation of $\mathcal{Y}_{\robotTraitMatrix[bold]}$ by considering sampling from the convex hull of the points within $\mathcal{Y}_{\robotTraitMatrix[bold]}$. We call this idea the \textit{convex hull-constrained strategy}, or convex-hull strategy for short. We define $\mathbf{S} \in \mathcal{S} = [0,1]^{M\times N}$ as the relaxed continuous correspondent of allocation matrix \allocation. By sampling a $\mathbf{S}$ randomly and computing $\mathbf{Y} = \mathbf{S}\robotTraitMatrix[bold]$, we obtain a random sample within the convex hull mentioned above. Picking the sample to label is essentially solving,
\begin{equation}
\begin{aligned}
    \max_{\mathbf{\mathbf{Y}},~\mathbf{S}} \quad&\ V(\mathbf{Y}) \\
    s.t.\quad &\mathbf{Y}  = \mathbf{S} \mathbf{\robotTraitMatrix},\\ 
    &\mathbf{S} \in \mathcal{S}.
\end{aligned}
\label{eq:convex_hull_opt}
\end{equation}

We provide a short proof that the sample $\mathbf{Y}$ lies within the convex hull in Appendix.~\ref{appendixA}. To aid in understanding the thread of thoughts that connects all four strategies together, we provide a visualization of their sampling regions of interest in Fig.~\ref{fig:sampling_relaxation}, showing \textit{the trait space of a single task}. The visualization is for the allocation problem involving 3 robots, while each robot is characterized by two traits (i.e., corresponding to $x$ and $y$ axes on a 2D plane). We arrange the four strategies in the order of increasing realizability, from least realizable in Fig.~\ref{fig:sampling_relaxation}(a) (unconstrained) to perfectly realizable in Fig.~\ref{fig:sampling_relaxation}(d) (exact). Between the two extremes, Fig.~\ref{fig:sampling_relaxation}(b) and Fig.~\ref{fig:sampling_relaxation}(c) exemplify that honoring realizability can manifest as controlling the subspace to sample from. 

\begin{figure}
    \centering
    \includegraphics[width=0.75\linewidth]{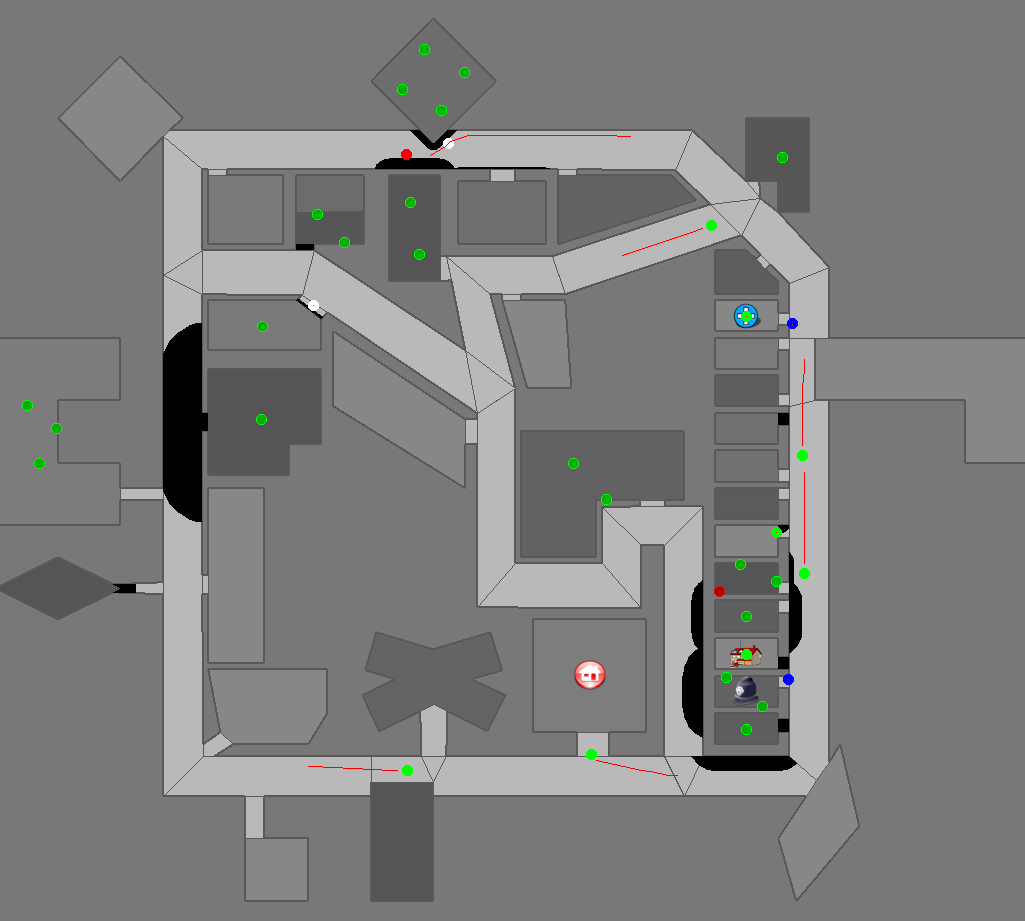}
    \caption{The layout of a RoboCup search and rescue scenario. Black obstacles indicate that pathways are (partially) obstructed. Green dots represent alive civilians waiting for rescue. The red dots are fire brigades, blue dots are police forces, and white dots are ambulance vehicles. }
    \label{fig:robocup}
\end{figure}

An overview of our active learning pipeline with realizability awareness is shown in Fig.~\ref{fig:active_learning_pipeline}. Given the team composition encoded by \robotTraitMatrix[bold], the constrained sampler randomly samples a set of trait combinations within either the convex hull or the box-shaped region to initialize the candidate pool. Then, the active learning procedure loops through the following steps until its budget to query the simulator runs out: (i) The adaptive neighborhood sampler samples a small number of points around each candidate within its radius. We call these points neighbors. (ii) The Gaussian process model evaluates each candidate by aggregating the UCB scores of its neighbors; (iii) The UCB-based selector selects a candidate with maximum aggregate score; (iv) The selected candidate is projected to its nearest neighbor in $\mathcal{Y}_{\robotTraitMatrix[bold]}$, termed as the realizable query. (v) The simulator accepts the allocation matrix \allocation~corresponding to the realizable query, simulates the task, and returns a score. 
(vi) The GP model is updated with the new pair of realizable query and score. 

Note that across our unconstrained, box, and convex hull strategies, we adopt the \textit{zooming mechanism} from~\cite{srinivas2009gaussian,CMTAB} to adaptively ``zoom-in'' on promising $\mathbf{Y}$ samples with a high utility score. The sampling strategy starts with a random set of candidate collective traits. Each candidate is associated with a local neighborhood, and we evaluate the utility of each candidate by aggregating the utilities of a set of neighbors within the neighborhood. The size of the neighborhood for each candidate is dynamically adjusted as the active learning process goes, with more frequent selections leading to a shrunk neighborhood. Such an adaptive mechanism allows us to zoom in on promising candidates and scrutinize them with finer granularity. Note that~\cite{CMTAB} implemented the zooming algorithm via discretizing the trait space into uniform grids. For our convex-hull-constrained strategy, generating uniform lattices with a specified number of on-grid points within a convex polytope by itself is a challenging problem. Therefore, we replaced the grid-based discretization with random uniform sampling while preserving the core idea of adaptive discretization. 

\section{Experiments and Results}
\label{sec:evaluations}


\begin{figure*}[t]
\begin{center}
    \includegraphics[width=0.9\textwidth, trim=0 0 0 0, clip]{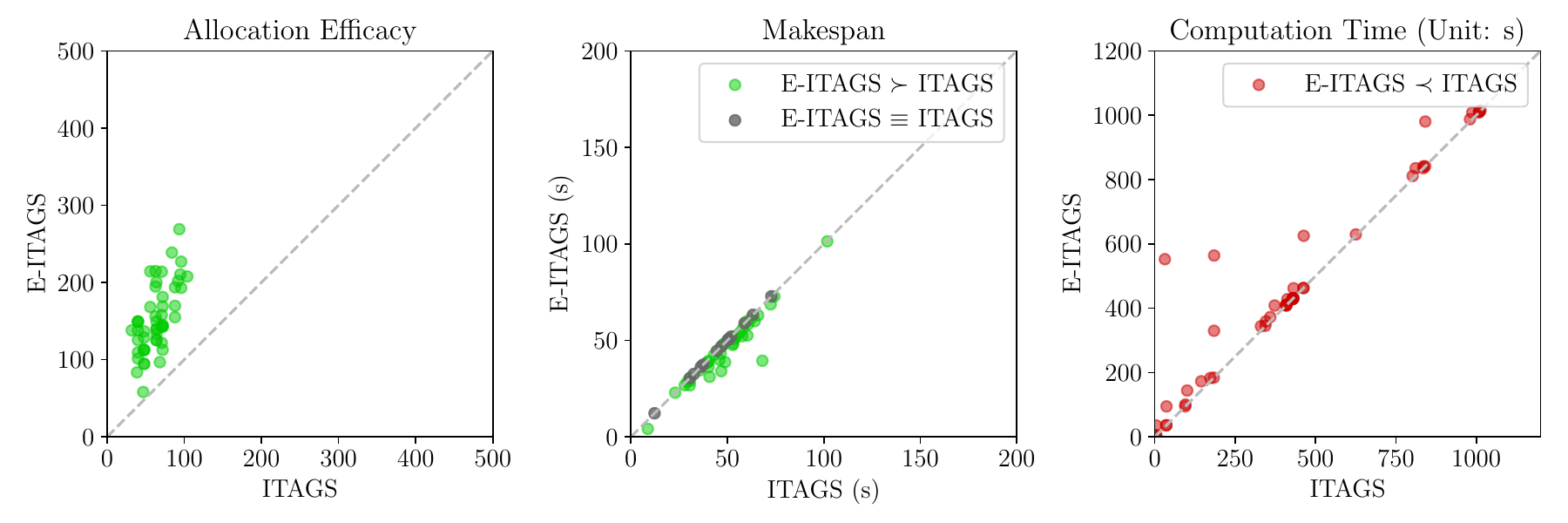}  
\end{center}
\caption{
\small{
Comparison of \algoname' allocation efficacy, makespan, and computation time against ITAGS. 
\algoname~consistently generates solutions of superior efficacy (left) while simultaneously ensuring that its makespan is better than or equal to that of ITAGS (middle). Green dots indicate \algoname{} performs better than ITAGS and grey dots indicate \algoname{} and ITAGS perform equally. Larger allocation efficacy and smaller makespan are desirable. Red dots indicate that \algoname{}' benefits over ITAGS comes at the cost of slightly worse computation time (right).
}
}
\label{fig:comp}
\end{figure*}

We evaluated \algoname{} using three sets of experiments in a simulated emergency response domain~\citep{Kitano1999,Bechon2014,Messing2020,Whitbrook2015,Zhao2016}. First, we compared \algoname{} against a state-of-the-art trait-based task allocation method regarding allocation performance. Second, we examined if the theoretical bounds on resource allocation suboptimality hold in practice. We went on to evaluate the realizability-aware active learning module in \algoname{}. The evaluation is first conducted via numerical simulations, followed by evaluating the learned models holistically as part of \algoname{}' heuristic search. 

\subsection{Experimental Domain}
\label{subsec:Experimental_domain}
We first provide a brief introduction to the experiment domain that we used throughout the evaluations. For our experiments in Sec.~\ref{subsec:ITAGS_comparison} and Sec.~\ref{subsec:validation_makespan}, we adopt the RoboCup Search and Rescue scenario~\citep{kitano2001robocup,Kitano1999,robocupwebsite} as a running domain. 
In this simulated domain, a heterogeneous team of robots, involving police force, ambulance vehicles, etc., is tasked to search for and rescue civilians after disasters such as fires have happened. 
We show a screenshot of this domain during one search and rescue mission in Fig.~\ref{fig:robocup}. 
Additionally, in Sec.~\ref{subsec:holistic_eval}, we further rely on the RoboCup Rescue simulator to provide ground-truth labels to our active learning pipeline. 

\subsection{Evaluating efficacy Optimization under Time Budgets}
\label{subsec:ITAGS_comparison}

We evaluate \algoname' performance on 50 problem instances, and contrast it with ITAGS~\citep{Neville2021,neville-2023-d-itags}. We choose to compare our approach with ITAGS as i) it has been shown to perform better than other state-of-the-art time-extended task allocation algorithms, and ii) ITAGS is also a trait-based approach that shares similar assumptions allowing for a more reasonable comparison. A key difference is that ITAGS minimizes makespan subject to thresholded allocation requirements, while \algoname{} optimizes allocation performance subject to maximum-makespan requirements. Considering that ITAGS is agnostic of the notion of trait-efficacy maps, we assume access to the ground truth of these functions, which are modeled as weighted linear combinations of coalitions' traits.
For each problem instance, we first compute a solution using ITAGS. The allocation efficacy of ITAGS' solution is computed using the ground truth trait-efficacy maps. 
We then evaluated \algoname' ability to improve ITAGS' allocation efficacy when constraining \algoname' time budget to be equal to ITAGS solution's makespan. 

\begin{figure}[b]
    \centering
    \includegraphics[width=0.97\linewidth]{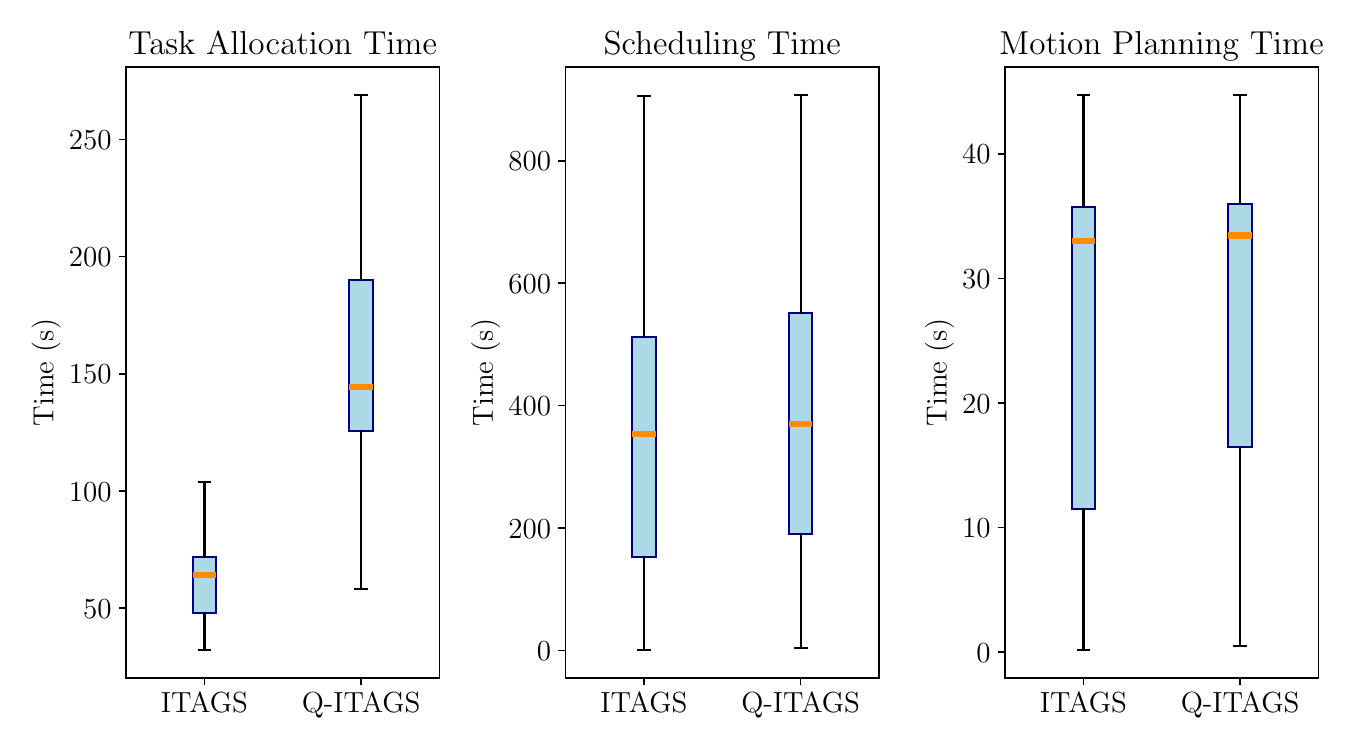}
    \caption{
    \small{We compare the runtime of \algoname{} and ITAGS across three crucial modules: task allocation, scheduling, and motion planning. The increase in total runtime can be attributed to \algoname{} spending more time searching through the task allocation graph. }
    }
    \label{fig:runtime_analysis}
\end{figure}

As shown in Fig. \ref{fig:comp} (\textit{left} and \textit{middle}), \algoname{} consistently improves the allocation efficay for all problems, while ensuring that its makespan is either equal to or lower than that of ITAGS. 
This is explained by the fact that \algoname{} can ensure that the makespan doesn't exceed that of ITAGS and yet search for a better solution that maximizes task performance. However, this improved performance comes at the cost of a slight increase in computation time as seen in Fig.~\ref{fig:comp} (\textit{right}). 

To understand where the increase in computation time comes from, we perform a more fine-grained analysis to compare the time \algoname{} and ITAGS spend on task allocation, scheduling, and motion planning, respectively. The results are summarized as box plots in Fig.~\ref{fig:runtime_analysis}. We observe that while both approaches spend comparable time on scheduling and motion planning, \algoname{} on aggregate requires more time for task allocation search, explaining the increased computational efforts in Fig.~\ref{fig:comp}. 
Since node evaluation makes the major portion of the amount of computation in task allocation search, we performed a Welch's \textit{t}-test on the number of nodes evaluated by \algoname{} and ITAGS to find the solution allocation plan. Our test gives a \textit{p}-value of $0.0047$, rejecting the null hypothesis and confirming that compared with ITAGS, \algoname{} statistically explores more nodes before it lands on an allocation plan whose makespan satisfies the time budget.  

\subsection{Validation of Bounds on Optimality Gaps}
\label{subsec:validation_makespan}

In our second experiment, we empirically examined the validity of our theoretical guarantees on allocation efficacy from Sec. \ref{sec:theorems}. To this end, we varied $\alpha$ between $[0,1]$ in increments of $0.1$, and solved each of the 50 problem instances from the previous section for each value of $\alpha$. For every combination of problem and $\alpha$ value, we computed the actual normalized optimality gap and the corresponding normalized theoretical bound in Eq.~(\ref{equ:boundsPre}). We found that the optimality gaps consistently respect the theoretical bound across all values of $\alpha$ (see Fig. \ref{fig:proof}). As we expected, the extreme values $\alpha=0$ (ignoring TBO) and $\alpha=1$ (ignoring NAC) respectively result in the best and worst allocation efficacy.

\begin{figure}[h]
\begin{center}
    \includegraphics[width=0.85\columnwidth, trim=0 0.6cm 0 0, clip]{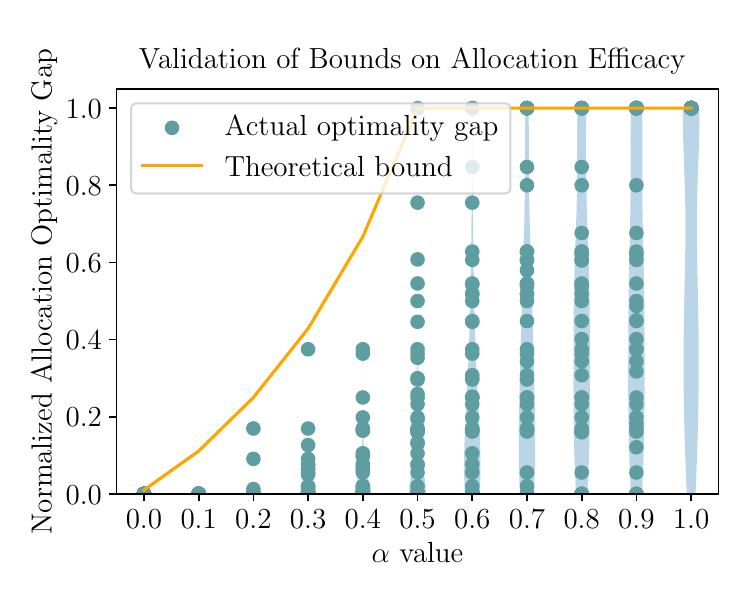}
\end{center}
\caption{
\small{
The theoretical bound consistently holds for varying values of $\alpha$. 
A value of $0$ for a normalized optimality gap represents an optimal allocation, and a value of $1$ represents the worst possible allocation seen within the experiments.
}}
\label{fig:proof} 
\end{figure}

\begin{figure}
    \centering
    \includegraphics[width=0.98\linewidth]{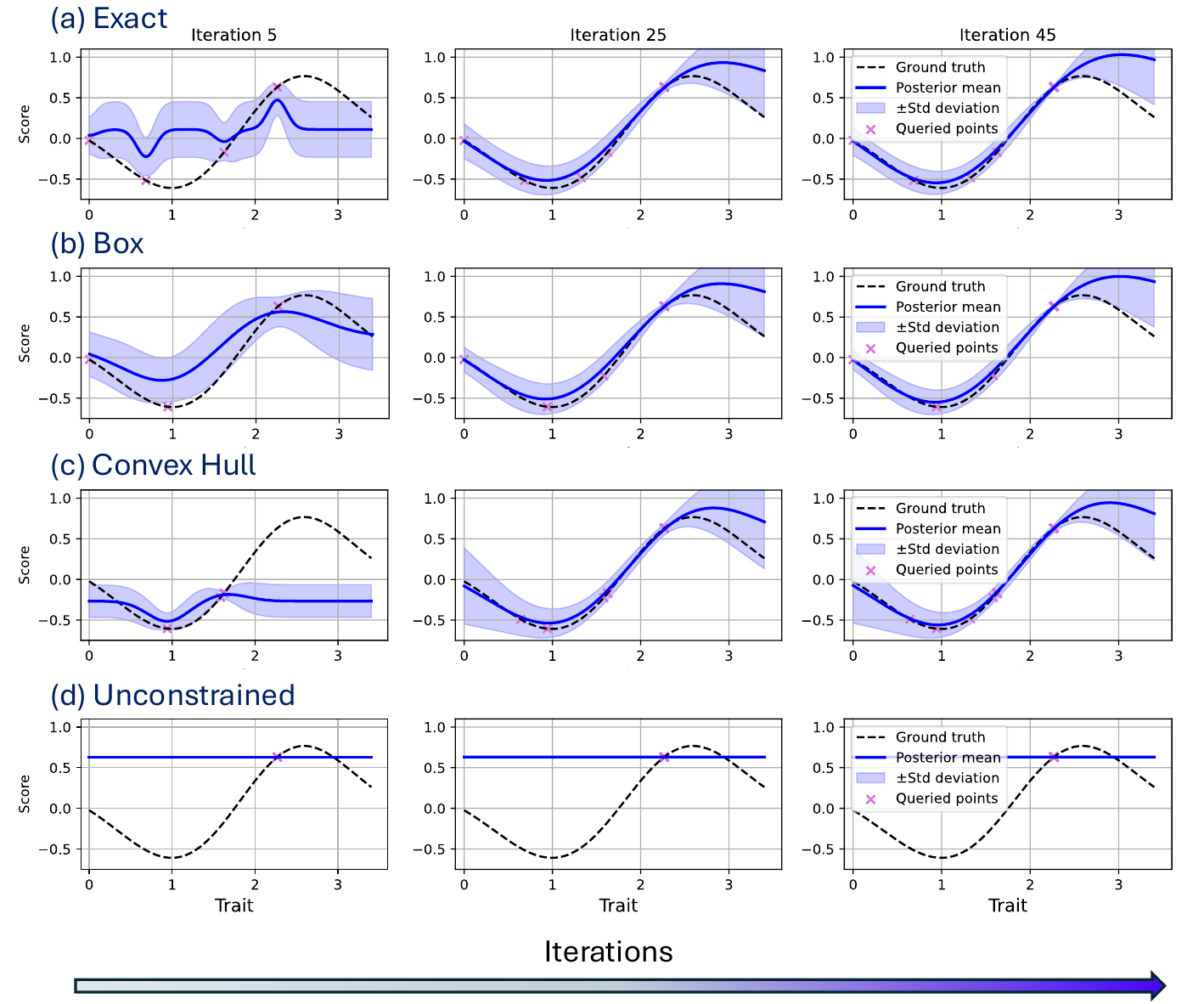}
    \caption{Visualization of the process of iteratively updating GPs in a simple scenario, where all strategies learn to allocate 3 robots to one task. Only one trait is considered for ease of demonstration. We show the ground-truth GP, fitted function ($\pm$ one standard deviation), and the sampled trait values. All but the unconstrained strategy are able to learn. }
    \label{fig:gps_1d}
\end{figure}

\begin{figure*}
\begin{center}
    \includegraphics[width=0.85\textwidth, trim=0 0 0 0, clip]{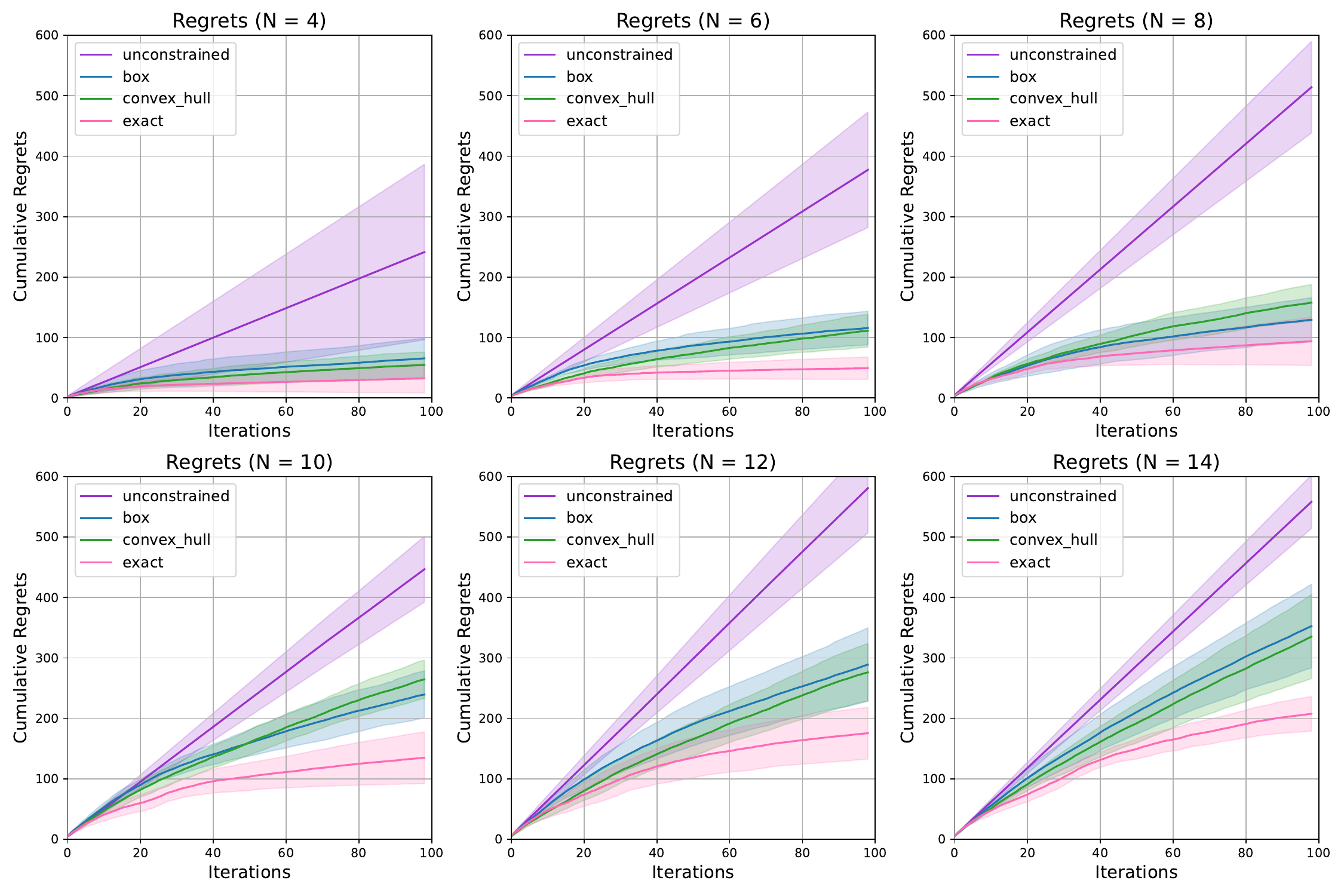}  
\end{center}
\caption{
Numerical simulation results comparing the four methods with a scaling number of robots. There are three tasks in this experiment. Robots and tasks are modeled by 4 different traits. We show the cumulative regrets averaged across 5 trials, with the shaded region indicating $\pm$ one standard deviation. 
}
\label{fig:cumu_regrets}
\end{figure*}
\subsection{Evaluating Realizability-Aware Active Learning}
\label{subsec:evaulate_active_learning}
In Sec.~\ref {subsec:active_learning_staq}, we proposed four different strategies for selecting the sample to query upon: unconstrained, exact, box-constrained, convex hull-constrained. We now evaluate these four strategies and compare their ability to learn the unknown trait-efficacy maps and uncover rewarding allocation plans. Our experiment results consist of: (i) numerical simulations with a simple 1D trait space that aim at illustrating the learning behaviors of different strategies qualitatively, (ii) comprehensive quantitative numerical simulations that show the regrets under different strategies,  
and (iii) evaluations on the RoboCup Search and Rescue simulator. 
For both (i) and (ii), the ground truth trait-efficacy mapping functions are sampled from Gaussian processes with covariances specified by Radial Basis Function (RBF) kernels. For (iii), we rely on the simulator to return a score as label given a coalition of robots and their capabilities. For all our experiments, the integer quadratic program in Eq.~\eqref{eq:projection_opt} is solved using \texttt{cvxpy}~\citep{diamond2016cvxpy,agrawal2018rewriting}. 


We first discuss results from numerical simulations of a simple scenario in which we allocate three robots (characterized by only one trait) to one task. The task's ground truth trait-efficacy map is sampled from a Gaussian process and kept fixed. The trait values for robots are sampled randomly from $[0, 1]$. To shed light on different strategies' learning behaviors throughout the learning process, we visualize the queried trait vectors $\mathbf{Y}$ and the fitted Gaussian process model, which are both progressively updated, across varying iterations. The results are shown in Fig.~\ref{fig:gps_1d}. We show the mean of GP model as solid blue curve, its variance as blue shades, the ground truth trait-efficacy map as dashed black curve, and the queried trait vectors as purple crosses. Notably, the unconstrained strategy almost learns nothing even after 45 iterations, while all the other three realizability-aware strategies can effectively learn the underlying function. The unconstrained strategy also tries out distinctly fewer trait combinations, implying its incompetence in exploration. 
This reveals that the unconstrained strategy, while capable of sampling diverse set of points, unproductively explores
the entire
trait space. The sampled points might be far away from the nearest realizable ones. From the observed harmful effects, we uncover that allowing the sampling strategy to remain oblivious of realizability induces 
\textit{unproductive diversity},
impeding efficient learning. 

With the qualitative learning process in mind, we comprehensively compare the four strategies on a suite of experiments involving 3 tasks characterized by 4 different traits. We vary the number of robots to allocate from 4 to 15. Given that the ground truth trait-efficacy maps for tasks are sampled and known, we can pre-evaluate all the possible allocation plans to obtain the ground truth optimal allocation for each task. Then, we can leverage it to compute the regret during active learning. In Fig.~\ref{fig:cumu_regrets}, we present the cumulative regrets averaged across 5 trials for a varying number of robots. We observe that across the board, the exact method consistently achieves the best performance with slowly-increasing cumulative regrets. The unconstrained strategy, in stark contrast, shows almost linearly increasing cumulative regrets, which indicate that it keeps querying data points that are not promising. 
This comparison reveals that choosing to be oblivious of sample's realizability while only compensating for it from hindsight by projection is very inefficient. 
The box and convex hull-constrained strategies induce significantly lower cumulative regrets compared with the unconstrained one, and achieve performance that is similar to the exact case, especially for small-size problems (with $\leq 8$ robots). It is evident that by injecting realizability awareness, the active learning pipeline attains a more reasonable exploration-exploitation tradeoff and better data efficiency. We also point out that across all strategies, a larger number of robots to allocate results in higher cumulative regrets. This stems from the fact that a larger combinatorial space naturally arises with a larger team size, and thus, the active learning for rewarding allocations becomes more challenging.

\begin{figure}
    \includegraphics[width=0.48\textwidth, trim=0 0 0 0, clip]{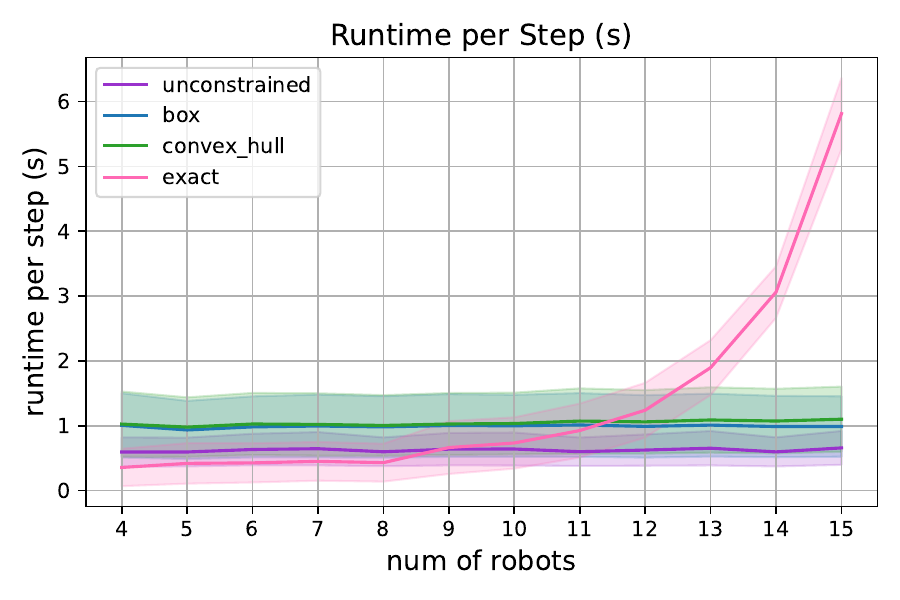}
\caption{The comparison of single-step runtime for the experiment shown in Fig~\ref{fig:cumu_regrets}. We show that while the runtime per step for exact strategy blows up as the number of robots increases, the runtime for all the other three strategies remains nearly constant. }
\label{fig:runtime}
\end{figure}

\begin{figure*}
    \centering
    \includegraphics[width=0.98\linewidth]{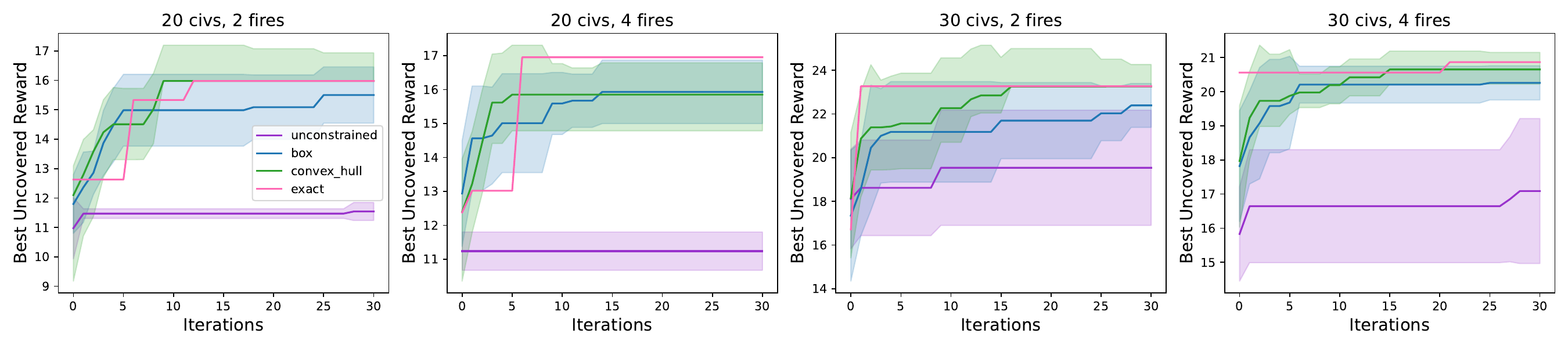}
    \caption{Comparing the evolution of Best Uncovered Reward (BUR) across the four constraining strategies on the RoboCup Search and Rescue simulator. The subfigures denote four different tasks, differentiated by their severity levels dictated by the number of civilians to save and the number of fires. For each task and strategy, we run the active learning module 5 times and report the mean as solid curves and the standard deviation as shaded regions. }
    \label{fig:robocup_uncovered_rewards}
\end{figure*}
\begin{figure*}
    \centering
    \includegraphics[width=0.95\linewidth]{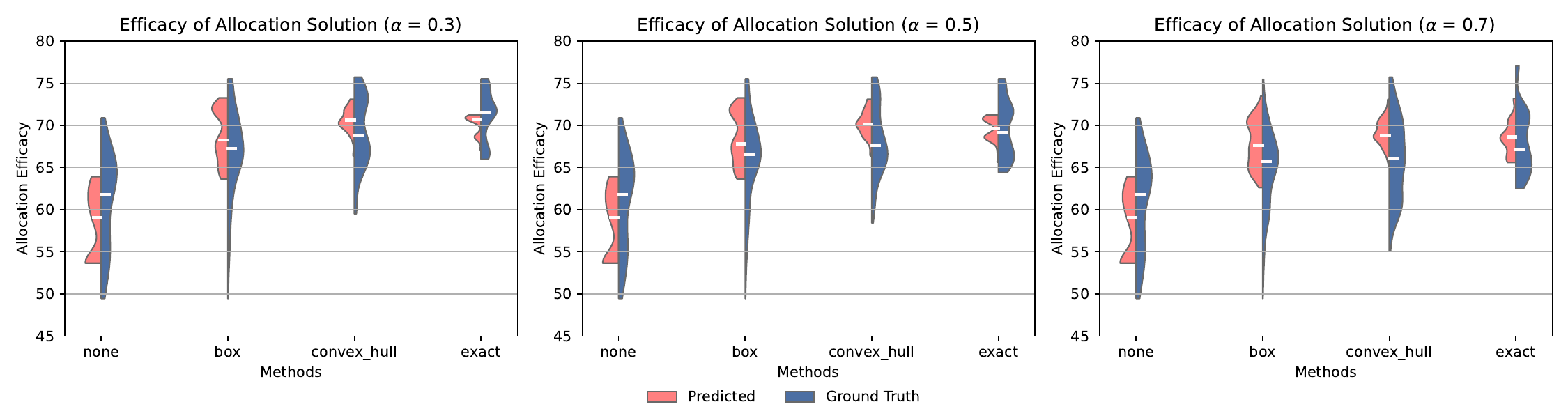}
    \caption{Total allocation efficacy of allocation plans when plugging in the learned models into downstream graph search with varying weights $\alpha=0.3, 0.5$ and $0.7$. All red violin plots come from efficacy scores predicted by the models, while all blue violin plots are obtained from groundtruth. White horizontal bars denote the median efficacy. }
    \label{fig:end2end_evaluation}
\end{figure*}

We additionally compared the four strategies' single-step runtime (See Fig.~\ref{fig:runtime}). While unconstrained, box, and convex hull strategies maintain a nearly constant single-step runtime, the time it takes for the exact strategy increases drastically once the number of robots exceeds 10. When there are 15 robots to allocate, it takes the exact strategy almost $6$ seconds to run a single step. The sharp increase in computation time originates from the exact strategy's mechanism of checking every possible allocation and its corresponding trait combination. To evaluate a trait combination, the exact strategy has to query its GP model. Since the number of possible allocations increases exponentially with the number of robots, the accumulation of local GP queries leads to the observed heavy computation burden. 
This issue is sidestepped by the other three approaches since they evaluate a fixed number of candidates at each iteration. When combined with the results shown in Fig.~\ref{fig:cumu_regrets}, it is clear that the exact strategy is preferable for small-scale problems with only a few robots. As the team size scales up, however, the box-constrained and convex hull-constrained strategies become a more reasonable choice thanks to their computational efficiency and comparable learning efficiency and effectiveness. 

\subsection{Holistic Evaluation}
In our final set of experiments, we evaluate the overall performance of our approach by learning the trait-efficacy maps and leveraging the learned maps to perform efficacy-optimized task allocation. 
\label{subsec:holistic_eval}
\subsubsection{Learning Trait-efficacy Maps in RoboCup}
We further verify our findings in Sec.~\ref{subsec:evaulate_active_learning} by interfacing the active learning module with the RoboCup Search and Rescue simulator~\citep{kitano2001robocup,Kitano1999,robocupwebsite}.
The RoboCup simulator allows us to change the number of civilians to rescue and the number of fires, which serve as surrogate indicators of task severity and difficulty. We consider three types of agents as the robots we aim to allocate: fire brigades, police forces, and ambulance vehicles. They are in charge of putting out fires, clearing blocked roads, and transporting civilians to hospitals, respectively. 
Since this simulator does not support parallel simulations, we simulate one task at a time and measure the assigned coalition's performance (i.e., efficacy) as measured by the score computed by the simulator, with higher scores representing better task performance. We simulate tasks of different severity levels separately.  

Since we no longer have access to the ground truth trait-efficacy maps, we report \textit{Best Uncovered Rewards} (BUR) as the metric to evaluate our four sampling strategies. BUR is the highest uncovered reward at each iteration, and reveals how efficient each method is in terms of identifying rewarding robot assignments. Fig.~\ref{fig:robocup_uncovered_rewards} shows the BUR results for 4 respective tasks. The comparative disadvantage of the unconstrained approach, which we previoiusly observe in numerical simulations, still holds here. Overall, the exact strategy consistently performs the best, discovering coalitions giving rise to high scores quickly. The box and convex hull-constrained strategies perform comparably to the exact one, though sometimes converge to suboptimal allocations or demonstrate larger BUR variances. 

\subsubsection{Leveraging Learned Trait-efficacy Maps for STEAM}
To measure the downstream implications of our active learning approach, we utilize the learned trait-efficacy maps from the previous subsection to compute the NAC heuristic in order to guide \algoname{}' incremental search. 
We randomize the robots' initial and task configurations across 30 problem instances. We vary both task durations and the allotted time budgets $C_{max}$. We also repeated this experiment thrice by choosing $\alpha = 0.3, 0.5, 0.7$ to evaluate how the learned efficacy models manifest within \algoname{} under varying degrees of relative prioritization of allocation efficacy and makespan. We show both ground truth and predicted total allocation efficacies as half violin plots in Fig.~\ref{fig:end2end_evaluation}. These plots enable us to analyze how approximating realizability constraints to different degrees in the interest of computational efficiency during learning impacts downstream task performance during task allocation. We can also see how accurate different approaches are in their predictions. 

As one would expect, we see that the exact sampling strategy gives rise to allocation solutions whose predicted efficacies closely match the ground truth. Further, the true allocation efficacies of the solutions associated with the exact method remain the highest across the board among the four sampling strategies. Importantly, the exact strategy is closely followed by the convex hull constrained and then by box constrained strategies, exhibiting slightly degraded total efficacies that can be viewed as the ``cost of approximation''. At the other end, the allocation plans generated by the unconstrained strategy are of the worst efficacy, indicating that learning trait-efficacy maps without any regard for practical realizability constraints can significantly reduce performance. Finally, these trends remain mostly robust as we vary the hyperparameter $\alpha$, with the differences between methods marginally reducing as we increasingly prioritize time budget over task allocation efficacy with a larger $\alpha$.

\section{Conclusion}
\label{sec:conclusion}


In this work, we introduced Spatio-Temporal Efficacy-optimized Allocation for Multi-robot systems (STEAM) problem, a new formulation that explicitly couples allocation efficacy with temporal feasibility. To address this challenge, we developed E-ITAGS, an interleaved graph search algorithm that integrates allocation, scheduling, and motion planning under strict time budgets. Unlike prior methods that rely on brittle assumptions about task decomposition or binary success criteria, E-ITAGS leverages trait-efficacy maps—learned through a realizability-aware active learning module—to capture nuanced, continuous relationships between collective capabilities and task performance.

Our contributions are threefold. First, we provide a generalizable formalism that unifies task allocation efficacy and temporal feasibility in heterogeneous multi-robot systems. Second, we design a computationally efficient interleaved search procedure guided by novel heuristics that balance allocation efficacy against makespan, with accompanying theoretical guarantees on suboptimality bounds. Third, we advance active learning for MRTA by introducing realizability-aware sampling strategies that significantly improve data efficiency and computational tractability in combinatorial domains. 

Comprehensive evaluations in simulation and an emergency response domain confirm that E-ITAGS consistently delivers allocation plans of higher performance than state-of-the-art baselines while respecting strict temporal and resource constraints. Beyond performance gains, our results demonstrate the practicality of integrating active learning into task allocation systems that must operate under limited knowledge of task requirements.

Collectively, these contributions position E-ITAGS as a robust and extensible framework for spatio-temporal coordination in heterogeneous multi-robot systems operating under stringent temporal and resource constraints. Future work will extend this framework to broader domains by integrating realizability-aware active learning with large-scale simulators and real-world deployments, thereby enabling adaptive, high-quality performance in dynamic and uncertain environments. 






\begin{acks}
This work was supported in part by the Army Research Lab under Grants W911NF-17-2-0181 (DCIST CRA).
\end{acks}




\bibliographystyle{SageH}
\bibliography{references.bib}

\clearpage
\begin{appendices} 

\section{Sampling from a Convex Hull}
\label{appendixA}
Recall from Sec.~\ref{subsec:active_learning_approach} that for the convex-hull constrained strategy, we sample $\mathbf{S}$ from $\mathcal{S}=[0,1]^{M\times N}$, then obtain the traits allocated to all tasks as $\mathbf{Y} = \mathbf{S}\robotTraitMatrix[bold]$. Here, we prove that the $\mathbf{Y}$ obtained this way lies within the convex hull of the realizable set of points in $\mathcal{Y}_{\robotTraitMatrix[bold]}$. 
For $M$ tasks, $\mathcal{Y}_{\robotTraitMatrix[bold]}$ contains $K = M\cdot 2^N$ vertices in the trait space, if $N$ robots are to be assigned. Note that $K$ is not exponential with respect to $M$ since the tasks do not necessarily happen concurrently. 
We denote these $K$ points as $\mathbf{Y}^1,\mathbf{Y}^2,\dots,\mathbf{Y}^{K}$. Correspondingly, we define $\allocation^k$ as the allocation matrix satisfying $\mathbf{Y}^k = \mathbf{A}^k \robotTraitMatrix[bold]$. 
Note that in the allocation space, there exists a convex hull for $\{\allocation^k\}_{k=1}^K$, which we call $\mathcal{CH}_\allocation$. 
By the nature of the relaxation technique, for any sampled $\mathbf{S}$, we know it lies within $\mathcal{CH}_\allocation$, since we can always find a set of non-negative real numbers $\lambda_1,\lambda_2, \dots, \lambda_K$ satisfying $\sum_{k=1}^{K}\lambda_k = 1$ and $\mathbf{S}=\sum_{k=1}^{K}\lambda_k \allocation^k$. Then, we know: 
\begin{equation}
    \begin{aligned}
        \mathbf{S}\robotTraitMatrix[bold] &= (\sum_{k=1}^{K}\lambda_k \allocation^k)\robotTraitMatrix[bold] \\
        & =\sum_{k=1}^{K}\lambda_k \allocation^k\robotTraitMatrix[bold] \\
        & =\sum_{k=1}^{K}\lambda_k \mathbf{Y}^k.
    \end{aligned} \label{eq:ch_proof}
\end{equation}
The left side of Eq.~\ref{eq:ch_proof} reflects how we sampled a point in the convex-hull constrained strategy, while the right side trivially resides in the convex hull of $\mathcal{Y}_{\robotTraitMatrix[bold]}$ by definition. We thus proved that generating $\mathbf{Y}$ as $\mathbf{Y} = \mathbf{S}\robotTraitMatrix[bold]$ ensures that it is inside the convex hull of $\mathcal{Y}_{\robotTraitMatrix[bold]}$.

\end{appendices}

\end{document}